\documentclass{article}

    \PassOptionsToPackage{numbers, compress}{natbib}

\usepackage[preprint]{neurips_2022}

\usepackage[utf8]{inputenc} 
\usepackage[T1]{fontenc}    
\usepackage{hyperref}       
\usepackage{url}            
\usepackage{booktabs}       
\usepackage{amsfonts}       
\usepackage{nicefrac}       
\usepackage{microtype}      
\usepackage{xcolor}         

\usepackage{algorithm}
\usepackage{bm}
\usepackage{algorithmic}
\usepackage{wrapfig}
\usepackage{siunitx}
\usepackage{amsmath,amssymb}
\usepackage{amsthm}
\usepackage{times,grffile}
\usepackage{color}
\usepackage{graphicx}
\usepackage{subfigure}
\usepackage{tablefootnote}

\newtheorem{theorem}{Theorem}
\newtheorem{proposition}{Proposition}
\newtheorem{lemma}{Lemma}
\newtheorem{assumption}{Assumption}
\newtheorem{definition}{Definition}

\title{Multiple Domain Causal Networks}

\author{
Tianhui Zhou$^{1}$, William E. Carson IV$^{2}$, Michael Hunter Klein$^{3}$,  David Carlson$^{1,3,4}$\\
$^1$Department of Biostatistics and Bioinformatics\\
$^2$Department of Biomedical Engineering\\
$^3$Department of Electrical and Computer Engineering\\
$^4$Department of Civil and Environmental Engineering\\
Duke University \\
Durham, NC 27708 \\  
\texttt{\{tianhui.zhou, wec14, michael.klein413, david.carlson\}@duke.edu}
}

\begin{document}

\maketitle

\begin{abstract}

Observational studies are regarded as economic alternatives to randomized trials, often used in their stead to investigate and determine treatment efficacy. Due to lack of sample size, observational studies commonly combine data from multiple sources or different sites/centers. Despite the benefits of an increased sample size, a na\"ive combination of multicenter data may result in incongruities stemming from center-specific protocols for generating cohorts or reactions towards treatments distinct to a given center, among other things. These issues arise in a variety of other contexts, including capturing a treatment effect related to an individual's unique biological characteristics. Existing methods for estimating heterogeneous treatment effects have not adequately addressed the multicenter context, but rather treat it simply as a means to obtain sufficient sample size. Additionally, previous approaches to estimating treatment effects do not straightforwardly generalize to the multicenter design, especially when required to provide treatment insights for patients from a new, unobserved center. To address these shortcomings, we propose Multiple Domain Causal Networks (MDCN), an approach that simultaneously strengthens the information sharing between similar centers while addressing the selection bias in treatment assignment through learning of a new feature embedding. In empirical evaluations, MDCN is consistently more accurate when estimating the heterogeneous treatment effect in new centers compared to benchmarks that adjust solely based on treatment imbalance or general center differences. Finally, we justify our approach by providing theoretical analyses that demonstrate that MDCN improves on the generalization bound of the new, unobserved target center.

\end{abstract}

\section{Introduction}

\label{intro}


Recent advancements in deep learning have facilitated characterization and modeling of complex relationships \citep{reed1999neural,schmidhuber2015deep}. Accordingly, deep learning has since been applied to problems of personalized medicine, with the goal to accurately quantify each individual's heterogeneous response to a given treatment \citep{jain2002personalized}. Due to the high cost of conducting experiments, researchers often turn to observational data to find clues through building predictive models or making quantitative assessments \citep{deaton2010instruments}. To enlarge the overall sample size, data collected from multiple different centers (clusters) are combined. This practice is commonly referred to as a ``multicenter observational study'' \citep{ferrer2009effectiveness}. Na\"ively combining data from multiple sources or ignoring the underlying cluster structure can be problematic. For example, individuals from different centers could have distinctive feature distributions with minimum overlap. Moreover, the mechanisms underlying treatment selection bias may differ from center to center. The combination of these issues present a challenge to researchers attempting to quantitatively model the effects of treatment assignment of patients from a new center. Limited work has been conducted for the causal estimation in a multicenter setting \citep{suk2021random}, as the vast majority of existing approaches address only part of the problem. To bridge this gap, our goal is to design a framework that jointly addresses the structural considerations of multicenter data \textit{and} robust estimation of heterogeneous treatment effects for individuals from new centers.

Different centers or different treatment groups within a same center could mathematically be described as feature spaces $X\in \mathcal{X}$ of different distributions. This aligns with the definition of ``domain'' in the domain adaptation literature where the similarity among different domains are leveraged to improve a model's generalization \citep{ben2010theory}. Inspired by the idea, we propose Multiple Domain Causal Networks (MDCN): a novel approach that takes into account both \textit{selection bias in treatment assignment} and \textit{discrepancies inherent to data collected at different centers}. Succinctly, we employ domain adaptation to simultaneously increase the overlap between treatment groups and match centers with more similarities to gain robustness in inference. Moving forward, we will use ``domain'' to refer to a center or cluster from which data was collected. Not only can the proposed approach be applied to data collected from multiple medical centers, but it can also can be used in broader contexts (e.g., considering individuals as separate domains). For instance, we motivate our method using neural data collected from mice, where behavioral assays can be framed as treatment regimes and we expect heterogeneous treatment effects between individual mice. This approach is highly relevant to observational studies that evaluate the impact of pharmaceuticals \cite{drysdale2017resting} or other interventions based on brain measurements (e.g. neurostimulation \cite{de2015neurostimulation}). Our main contributions include the following:
\begin{enumerate}
    \item We formulate the heterogeneous treatment effect estimation problem under the context of multicenter observational studies.
    \item We propose a new representation learning approach that accounts for both domain-level discrepancies and selection biases in treatment assignment.
    \item We provide supporting theoretical proof demonstrating that the error on the new, unobserved target center is bounded with the proposed method.
\end{enumerate}

\section{Related Work}

To date, few works address out-of-domain treatment effect estimation, especially on multicenter data. One example is generalized mixed effect models that handle cluster-structured data derived from multiple domains \citep{duckworth2010establishing,mcgilchrist1994estimation}; however, this approach is an additive linear model that requires specific likelihood functions for its mean model and variance model, which restricts customization and is less capable of handling heterogeneous relationships. Instead, we will build on recent advances in using machine learning for causal estimation and in domain adaption, briefly described below.

There are many recent advances in machine learning for causal inference. A variety of methods have been used, including methods that match individuals \citep{awan2020almost,chang2017informative,rubin2000combining} or parts of the covariate space (e.g. tree-based methods \citep{athey2016recursive,hill2011bayesian}. A critical issue in causal inference is mismatch of the covariate space, motivating weighting methods \citep{assaad2020counterfactual,li2018balancing}). Rather than weighting all samples, one could disentangle the confounding effect from the feature space \citep{kuang2017treatment,zhang2012bias}, find a more balanced embedding between treatment groups \citep{shalit2017estimating}, or embed the treatment assignment mechanism in a new latent representation \citep{hassanpourvariational2020,shi2019adapting}. Despite demonstrable improvement in model generalization for causal inference, these models do not account for multiple domain setups and their extensions to out-of-domain prediction are not straightforward.

Domain adaptation methods provide elegant solutions to account for the inherent distribution shift across different domains \cite{azizzadenesheli2018regularized,Sun2016,Tachet2020,Zhang2013}. Many domain adaptation methods aim to find a latent variable representation where the distributions of multiple domains are matched \cite{pmlr-v97-zhao19a}. Matched latent domains can be learned by penalizing with statistical measures such as the Maximum Mean Discrepancy \cite{Borgwardt2006} or Wasserstein Distance \cite{Shen_Qu_Zhang_Yu_2018}, or by matching second order statistics \cite{SunCoral2016}. Another family of techniques make use of adversarial methods to make the latent sample's source domain identity indistinguishable \cite{Ganin2016,Hoffman2018CyCADACA,Li2019}. It has been shown that too strong of a latent domain matching penalty can harm performance, especially when source domains are unrelated to each other \cite{Mansour2008} or suffer from label shift \cite{li2019target}. To account for this issue, several techniques have been developed for relaxing the latent domain matching penalty using other objective loss functions \cite{Hoffman2018CyCADACA} or by explicitly re-weighting the importance of matching domains through Multiple Domain Matching Networks (MDMN) \cite{limdmn}. MDMN uses a Wasserstein distance approximation to dynamically re-weight the importance of domain matching depending on the proximity of domains, choosing to use only the most relevant source domains. While these methods leverage advanced domain adaptation techniques, it is crucial in our situation to also examine the within-domain differences in treatment assignment. 

In summary, existing work only partially solves the problem of predicting out-of-domain treatment effects with data from multiple sources.

\section{Problem Setup}
\label{sec:statement}


Suppose we have data from $S$ domains (centers). Let $\{1, \cdots, S-1\}$ represent source domains, and let $S$ represent the unobserved target domain.
We assume a binary treatment condition with label $T\in\{0,1\}$. The feature space and the potential outcome space are represented as $X\in\mathcal{X}\subset\mathbb{R}^{p_x}$ and $\{Y(0),Y(1)\} \in\mathcal{Y}\subset\mathbb{R}^{p_y}$, respectively. For any domain $s$, the selection bias in observational studies causes $X$ to have distinct distributions in the control group ($T=0$), treatment group ($T=1$), and overall domain.
\begin{equation*}
    \text{Control: } X|_{T=0} \sim D_{s,0}; ~\text{Treatment: } X|_{T=1} \sim D_{s,1};  ~\text{Overall: }X \sim D_{s}.
\end{equation*}
We use $D_{s,0}, D_{s,1}$, and $D_s$ to represent their distribution functions, and their relationship can be described by $(1-p_s^{T=1})D_{s,0}+p_s^{T=1}D_{s,1}=D_{s}$, with $p_s^{T=1}$ indicating the marginal probability of treatment assignment. Likewise, $\{g_{s,0},g_{s,1}\}$ represents the unknown outcome functions for $\{Y(0),Y(1)\}$.

The observable information of an individual data sample of any source domain $s$ is represented via a tuple, $\{T,X,Y=Y(T)\}$. Given a target domain $S$ with $X \sim D_S$, our goal is to accurately infer both potential outcomes, so that we can more accurately estimate treatment outcomes and therefore assign individuals to the treatments that are likely to benefit them most.

There are two obstacles that are likely to hurt generalization to the target domain. The first is selection bias, since treatment assignment is often not randomized in observational studies. As a result, the acquired samples on which the model is trained cannot objectively reflect the error on the full domain, as $D_{s,0} \ne D_{s,1} \ne D_s, \forall s$. The second consideration is domain-level discrepancies, which may inhibit the effects of reducing the error over the combined source domains. More formally, for different domains $a$ and $b$, we could describe the discrepancies as the shifts in feature space distributions and underlying outcome functions, represented as, $D_{a} \ne D_{b}$ and $\{g_{a,0},g_{a,1}\} \ne \{g_{b,0},g_{b,1}\}$, respectively. Here, we make similar assumptions to those made in the causal literature \citep{hernan2020causal,rosenbaum1983central}; however, we make modifications that account for the extension to our multi-domain scenario.

\begin{assumption}[Within-domain positivity] For domain $s$ where $X \sim D_s$, the probability of assignment to any treatment group is bounded away from zero: $0<Pr(T=1|X ,s)<1$.
\label{assum:positivity}
\end{assumption}
\begin{assumption}[Within-domain consistency]
For domain $s$, the observed outcome with assigned treatment is equal to its potential outcome: $Y|T,X\sim D_s=Y(T)|T,X\sim D_s$.
\label{assum:consistency}
\end{assumption}
\begin{assumption}[Within-domain ignorability]
For domain $s$, the potential outcomes are jointly independent of the treatment assignment conditional on $X$: $[Y(0), Y(1)] \perp T|X \sim D_s$.
\label{assum:ignorability}
\end{assumption}

Considering treatment assignment as probabilistic ensures that studying the treatment difference forms a meaningful target for each domain. The original assumption does not involve $s$. This could create a problem where the treatment assignment is probabilistic overall ($0<Pr(T=1|X)<1$) but deterministic at domain level ($Pr(T=1|X,s)=1$ or $Pr(T=1|X,s)=0$). Assumptions \ref{assum:consistency} and \ref{assum:ignorability} are now also stated \textit{with respects to each domain} to better account for differences at the domain level. As a result, individuals with identical features but from different domains may have treatment assignments and responses that are domain-dependent and thus different overall.

\section{Multiple Domain Causal Networks}
\label{sec:mdcn}


Multiple Domain Causal Networks (MDCN) is a novel framework that both addresses selection bias in treatment assignment and leverages similarity from different centers. MCDN has three main components, all implemented as neural networks: a feature embedding network $\phi: ~ \mathbb{R}^{p_x} \rightarrow \mathbb{R}^q$ that learns new representations, and two outcome networks $h_0: ~ \mathbb{R}^{q} \rightarrow \mathbb{R}$ and $h_1: ~ \mathbb{R}^{q} \rightarrow \mathbb{R}$ that infer potential outcomes. The roles of these three components are depicted in Figure \ref{fig:mdcnframe}. Using the taxonomy in \citet{kunzel2019metalearners}, our outcome networks can be considered as a ``T-learner''. 

First, the original feature space $X$ is embedded in $\phi(X)$. This embedding is incorporated into three loss terms, the purpose of each being to endow it with specific properties. Through $L(\phi,h_0,h_1)$, $\phi$ is injected with information predictive of the outcomes. $L_{BT}(\phi)$ and $L_{CD}(\phi)$ are two regularization terms. The former makes $\phi$ more robust against the selection bias between treatment groups, and the latter reduces the cross-domain differences so that generalization is improved by learning from examples from similar domains.
\begin{equation}
L_{all}(\phi,h_0,h_1)=L(\phi,h_0,h_1)+\alpha L_{BT}(\phi)+\beta L_{CD}(\phi)
\label{eq:fullloss}
\end{equation}

\begin{wrapfigure}{r}{0.30\textwidth}
\vspace{-.20cm}
  \begin{center}
    \includegraphics[width=0.30\textwidth]{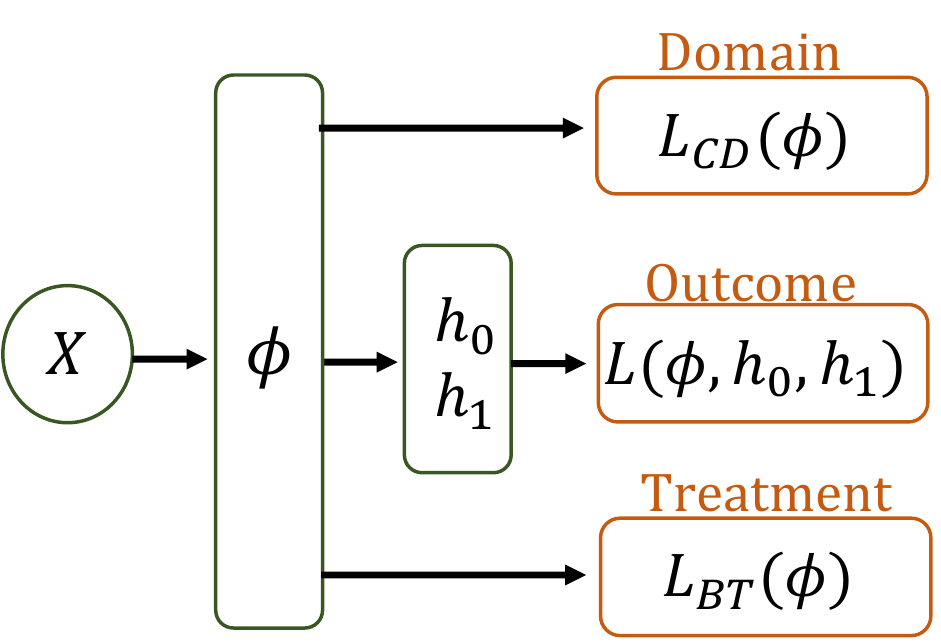}
  \end{center}
  \caption{MDCN, $\phi,h_0, h_1$ are optimized together with respective losses.}
  \vspace{-.03cm}
    \label{fig:mdcnframe}
\end{wrapfigure}
The full loss is summarized in \eqref{eq:fullloss}, with $\alpha$ and $\beta$ as tuning parameters.  In all our experiments, we fix $\alpha=$5e-4 and  $\beta$=1e-3, as chosen from the range of recommended values in \citet{shalit2017estimating} and \citet{limdmn}. Our model has empirically robust performance with these values, although they could be tuned in practice. We give the form of $L(\phi,h_0,h_1)$ below. The design of $L_{BT}$ and $L_{CD}$ are explained in Sections \ref{sec:lbt} and \ref{sec:lcd}. $L(\phi,h_0,h_1)$ trains the potential outcome models $h_0$ and $h_1$ to predict the observed outcomes given their assigned treatments. We use $l(\cdot,\cdot)$ to represent an arbitrary loss function, the form of which can be chosen based on the requirements of the modeling problem (e.g., cross-entropy loss for binary or categorical outcomes, squared loss for continuous outcomes).  \eqref{eq:outcome} represents the expectation of this loss as a summation over all source domains,
\begin{equation}
    L(\phi,h_0,h_1)=\sum\nolimits_{s=1}^{S-1} \biggl[E_{x\sim D_{s,0},y,t }l(h_0(\phi(x),y(t=0))+E_{x\sim D_{s,1},y,t}l(h_1(\phi(x),y(t=1))\biggl]
    \label{eq:outcome}
\end{equation}
We note that $h_0$ and $h_1$ are used to predict the potential outcomes universally across all domains. This strategy is more scalable compared to the alternative of learning a unique function for each domain. Additionally, we believe that when given enough capacity, $h_0$ and $h_1$ can still discern the differences at the domain level when combined with the extra properties endowed in $\phi$ learned via the losses $L_{BT}$ and $L_{CD}$. However, when we have a small number of domains, we could construct domain-specific $h_{0,s}$ and $h_{1,s}$. The prediction on the target domain could be based on a weighted summation of the outcome functions from its neighboring sources domains. We also view it as a future goal of varying the design of the potential outcome functions.
In Appendix \ref{app:model}, Algorithm \ref{alg} describes the pseudo-code for the full implementation.

\subsection{Between-Treatment Adjustment}
\label{sec:lbt}

Our approach to between-treatment adjustment is motivated by counterfactual regression (CFR) \citep{shalit2017estimating}. An embedding $\phi(x)$ with increased overlap between different treatment groups is learned to reduce the effects of selection bias, as supported by the theory of \citet{ben2010theory} that generalization between more similar spaces have a lower error bound.
To measure and encourage this overlap, we penalize with the Kantorovich-Rubinstein form of the Wasserstein-1 distance \citep{villani2008optimal}. 
\begin{equation}
\label{eq:wass}
    \sup_{\|{f_{bt}}\|_{L} \leq 1} \mathbb{E}_{x \sim {D}_{0}}[f_{bt}(\phi(x))]-\mathbb{E}_{x \sim {D}_{1}}[f_{bt}(\phi(x))]
\end{equation}
$D_0$ and $D_1$ in \eqref{eq:wass} characterize how $X$ is distributed on the treatment group and control group. $f_{bt}$ is drawn from the family of Lipschitz functions with constant 1, which can be approximated by neural networks with a gradient penalty \citep{gulrajani2017improved}. Through this objective, $\phi(x)$ will be encouraged to reduce the distance between-treatment groups. However, this only resolves the imbalance within a single domain. The direct application of CFR to balance all samples does not address the difficulties inherent to data collected from multiple domains. While overall balance is encouraged and enforced, the balancing of specific domains may be neglected. Moreover, the relative closeness of domains in the original space could also become distorted with this crude adjustment. As a modification, we adopt a separate ${f_{bt}^s}$ for each domain. This aims at reducing the distance between-treatment groups for \textit{all} domains. As ${f_{bt}^s}$ is domain-specific, the resulting adjustment is therefore also domain-specific. Additionally, there is less impact on cross-domain relationships. In Section \ref{sec:syn}, we provide visualizations to showcase the advantage of this approach over CFR. With this new strategy, the between group loss is,
\begin{equation}
\label{eq:bt}
    L_{BT}(\phi,f_{bt})=\sum\nolimits_{s=1}^{S-1} \biggl[ \sup _{\|{f_{bt}^s}\|_{L} \leq 1} \mathbb{E}_{x \sim {D}_{i,0}}[f_{bt}^s(\phi(x))]-\mathbb{E}_{x \sim {D}_{i,1}}[f_{bt}^s(\phi(x))] \biggl].
\end{equation}
We represent $L_{BT}(\phi,f_{cd})$ with notation $L_{BT}(\phi)$, as $f_{bt}$ is an auxiliary function to learn distances.
We combine $f_{bt}=\{f_{bt}^1,\cdots, f_{bt}^{S-1}\}$ using a single neural network with $S-1$ outputs, which saves computational cost and encourages information sharing between  domains, and use the methods of \citet{gulrajani2017improved} for learning.

\subsection{Cross-Domain Adjustment}


\label{sec:lcd}
\begin{wrapfigure}[17]{r}{0.31\textwidth}
  \begin{center}
  \vspace{-5mm}
    \includegraphics[width=0.31\textwidth]{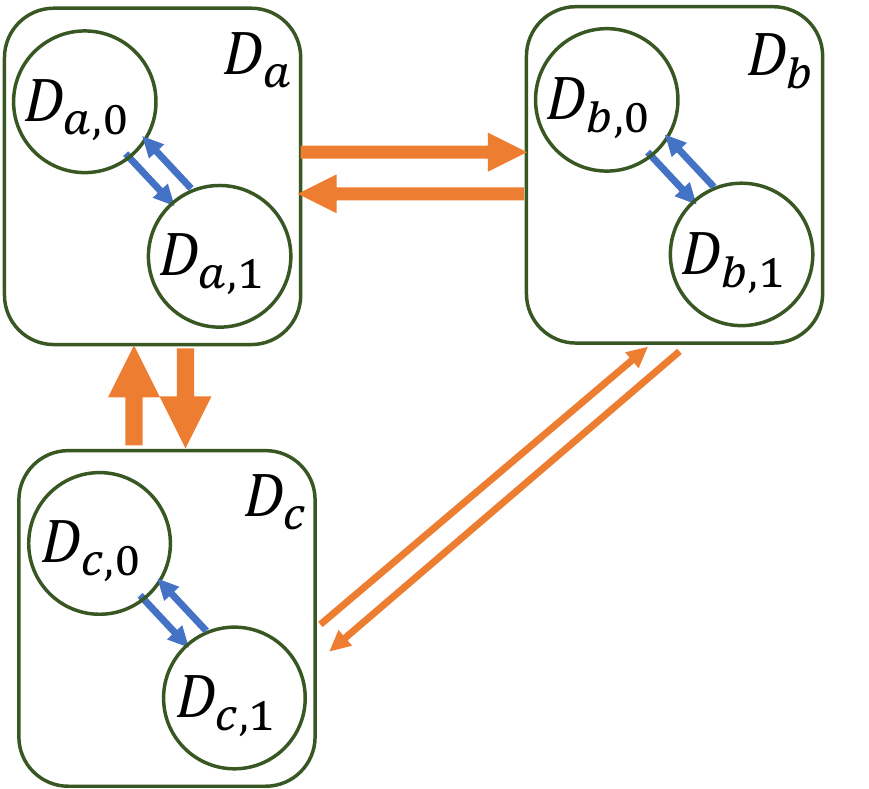}
  \end{center}
  \caption{Depiction of cross-domain adjustment with three domains. Wider arrows indicate larger weights/stronger relationships for similar domains.}
    \label{fig:3d}
    \vspace{-.6cm}
\end{wrapfigure}

We implement a cross-domain adjustment to encourage similarities at the domain level. This is achieved by extending the Multiple Domain Matching Network (MDMN) \cite{limdmn}. Akin to the between-treatment adjustment, we measure and reduce the distance based on the Wasserstein-1 distance. Instead of enforcing the distance penalty statically between any two domains, MDMN adopts a weighting scheme which assigns larger weight to more similar domains in \eqref{eq:sdom}. For domains that are far apart or dissimilar, forcefully reducing their differences may result in excessive loss of predictive information. As part of the predictive information is shared by all domains, the rest may be specific to only a few. The domain-specific information is reflected by these differences, which cannot be covered by the domain-invariant representation alone but is helpful in capturing the shift in outcome functions. Therefore, smaller weights are allocated in these instances. Ideally, each domain is matched to a few highly-correlated domains in space and forms clusters. For any domain $s$, this is formulated as,
\begin{equation}
\label{eq:sdom}
    d(D_{s},D_{/s})= \sup _{\|{f_{cd}^s}\|_{L} \leq 1} \mathbb{E}_{x \sim {D}_s}[f_{cd}^s(\phi(x))] -\sum\limits_{i\ne s} \mathbb{E}_{x \sim {D}_{i}}[w_{s,i}f_{cd}^s(\phi(x))].
\end{equation}
Specifically, the weight $w_{s,i}$ is calculated according to the pairwise differences between any two domains in \eqref{eq:rawdif}, and $\sum_{i\ne s}w_{s,i}=1$. While this distance is minimized by picking the single closest domain, we encourage robustness by smoothing over nearby domains.  
The weighted summation of all other domains can be seen as a pseudo domain, $D_{/s}$, with the well-defined distribution, $D_{/s}=\sum\nolimits_{i\ne s}w_{s,i}D_i$:
\begin{equation}
\label{eq:rawdif}
\begin{split}
&\boldsymbol{l}_s=\{l_{s,i}\}_{i\ne s},~
l_{s,i}=\mathbb{E}_{x \sim {D}_{s}}[f_{cd}^s(\phi(x))]-\mathbb{E}_{x \sim {D}_{i}}[f_{cd}^s(\phi(x))],\\
&\boldsymbol{w}_s=\{w_{s,i}\}_{i\ne s}=\text{softmax}(-\boldsymbol{l}_s).
\end{split}
\end{equation}
In practice, $\boldsymbol{l}_s$ is often accompanied with a temperature term to increase its stability. $\boldsymbol{l}_s$ at $t'$th iteration could be represented as: $\boldsymbol{l}_s^t=.9\boldsymbol{l}_s^{t-1}+.1\boldsymbol{l}_s^c$, with $\boldsymbol{l}_s^c$ being the estimates from the current batch \citep{limdmn}. Aggregating over all domains gives us the cross domain adjustment term:
\begin{equation}
\label{eq:cd}
L_{CD}(\phi,f_{cd})=\sum\limits_{s=1}^{S}\bigg[  \sup _{\|{f_{cd}^s}\|_{L} \leq 1} \mathbb{E}_{x \sim {D}_s}[f_{cd}^s(\phi(x))]-\sum\limits_{i\ne s} \mathbb{E}_{x \sim {D}_{i}}[w_{s,i}f_{cd}^s(\phi(x))] \bigg].
\end{equation}
Likewise, we can also use notation $L_{CD}(\phi)$ for $L_{CD}(\phi,f_{cd})$.
To control for the computational cost, we implement $f_{cd}=\{f_{cd}^1,\cdots, f_{cd}^{S}\}$ in the form of a multi-output function. However, due to the pairwise comparisons in weight calculation, the complexity $\mathcal{O}(S^2)$ is needed as $S$ increases. We train this loss and between-treatment adjustment simultaneously on $\phi$ again using the gradient penalties of \citet{gulrajani2017improved}, which is depicted in Figure \ref{fig:3d}.

\section{Theoretical Analysis}
\label{sec:theory}


Here, we provide theoretical analyses to explain the design of MDCN. The overarching goal is to bound the error on the unlabeled target domain via labeled source domain data. Full proofs for the theoretical results presented in this section can be found in Appendix \ref{asec:proof}. We give a few preliminaries below before presenting the final bound. To quantify the difference between hypothesis (or outcome) functions, we state the probabilistic discrepancy defined in \citet{limdmn}:
\begin{definition}
[Probabilistic discrepancy]
\label{def:pdis}
For two hypotheses $h$ and $h':\mathbb{R}^P\rightarrow \mathbb{R}$, their difference given a probabilistic distribution $D$ over $\mathcal{X}$ is defined as  $\gamma(h,h'|D)=\mathbb{E}_{x\sim D}|h(x)-h'(x)|$.
\end{definition}
Next, we define a family of our hypothesis functions in Definition \ref{def:lip}. We limit our proposed hypothesis class to the Lipschitz family with parameter $\lambda$. $\{h_0, h_1\} \in F_\lambda$.


\begin{definition}[Lipschitz continuity]
\label{def:lip}
A function $f :\mathbb{R}^P\rightarrow\mathbb{R}$ is Lipschitz continuous with parameter $\lambda$ if 
$|f(x_1)-f(x_2)|\leq \lambda ||x_1-x_2||_2$
holds for any vectors $x_1,x_2\in\mathcal{X}$. We denote the family as  $\mathcal{F}_\lambda$.
\end{definition}
We assume that the true outcome functions are also included in this family but with a different smoothness parameter $\lambda^*$,  $\{g_{s,0},g_{s,1}\}_{s=1}^{S} \in F_{\lambda^*}$. This encourages smooth transitions in the outcome labels with regards to changes in the feature space. Though smoothness may not hold in practice, Lipschitz continuous functions provide a reasonable approximation to a wide range of functions \citep{sohrab2003basic}. With all the preliminaries, we can present the overall bound on the target domain in Theorem \ref{thm:bound},
\begin{theorem}
\label{thm:bound}
For any positive weights $w=\{w_s\}_{s=1}^{S-1}$ with $\sum_{s=1}^{S-1} w_s=1$, the discrepancy between the true hypothesis functions $\{g_{0,S},g_{1,S}\}$ and the proposed hypothesis functions $\{h_0,h_1\}$ on target domain $S$ is bounded by,
 \begin{equation}
 \label{eq:bound}
\resizebox{.88\textwidth}{!}{$
 \begin{aligned}
      \gamma(h_0,g_{S,0}|D_S)+\gamma(h_1,g_{S,1}|D_S)&\le
     (\lambda+\lambda^*)[2W_1(D_S, \sum\nolimits_{s=1}^{S-1}w_{s}D_s)+\sum\nolimits_{s=1}^{S-1}w_{s}W_1(D_{s,0},D_{s,1})] \\
     &+
     \sum\nolimits_{s=1}^{S-1}w_{s}[\gamma(h_0,g_{s,0}|D_{s,0})+\gamma(h_1,g_{i,1}|D_{s,1})]+\gamma_0^*+\gamma_1^*.
 \end{aligned}
 $}
 \end{equation}
\end{theorem}
$[\gamma(h_0,g_{s,0}|D_{s,0})+\gamma(h_1,g_{i,1}|D_{s,1})]$ represents the probabilistic discrepancy between $\{h_0, h_1\}$ and the true outcome functions $\{g_{s,0}, g_{s,1}\}$ on the observable part of the data: $D_{s,0}, D_{s,1}$. Minimizing this corresponds to $L(\phi,h_0,h_1)$ in \eqref{eq:outcome}. In practice, we could modify the discrepancy mildly such as by optimizing the squared loss in regressions instead for more stability.
$W_1(D_{s,0},D_{s,1})$ measures the distance between the treatment space and the control space in domain $s$, and it matches with the structure of $L_{BC}(\phi)$. Likewise, $L_{CD}(\phi)$ is inspired by 
$W_1(D_S, \sum\nolimits_{s=1}^{S-1}w_{s}D_s)$. Though the bound is dependent on the values of $\lambda$ or $\lambda^*$, their explicit values are not needed in practice, as they can be incorporated into tuning parameters $\alpha$ and $\beta$. Lastly, $\gamma_0^*+\gamma_1^*$ depicts the fundamental difference in the true outcome functions $\{g_{s,0},g_{s,1}\}_{s=1}^{S}$ across domains, and cannot be optimized. A large $\gamma_0^*+\gamma_1^*$ signals the existence of a huge shift in outcome functions. In that case, we are less guaranteed to achieve accurate predictions on the target domain. The weight $w=\{w_s\}_{s=1}^{S-1}$ also plays a key role in this bound. In \eqref{eq:bound}, increased emphasis is placed on source domains that are close to the target domain. Samples from these similar source domains may serve as good reference cases and provide valuable insights in improving the performance on the target domain.

\section{Experiments}
\label{sec:exp}






We use two examples to demonstrate how MDCN improves estimation of heterogeneous treatment effects on the target domain. Efficacy is evaluated according to estimation of the conditional average treatment effect (CATE) \citep{shalit2017estimating}. We modify CATE by conditioning the quality on the domain label $S$ to take into consideration the shift in outcome functions:
\begin{equation}
\label{eq:CATE}
    \tau(x,S)=\mathbb{E}[Y(1)|x,S]-\mathbb{E}[Y(0)|x,S].
\end{equation}
We use the precision in estimation of heterogeneous treatment effects (PEHE) from \citet{hill2011bayesian} to measure the distance between CATE and its estimates: $ \text{PEHE(S)}=\sqrt{\mathbb{E}_{x\sim D_s}[\tau(x,S)-\hat{\tau}(x,S)]^2}$. As evaluating CATE requires access to ground truth values, we create synthetic and semi-synthetic data for this purpose. Causal forest (CF) is included as a benchmark due to its previous application to multicenter observational study data \cite{athey2016recursive,suk2021random}. Other benchmarks include different variants of MDCN, which also serves as an ablation study that helps us understand the benefits of different components of MDCN. Apart from CF, we note that all variants have identical architectures for the outcome model and feature embedding model to ensure a fair comparison. CFR is included as a na\"ive benchmark that only addresses treatment group imbalance globally \cite{shalit2017estimating}. The proposed between-treatment adjustment $L_{BT}(\phi)$ in \eqref{eq:bt} upgrades CFR by performing the novel adjustment \textit{within} each domain, which we call the domain CFR (DCFR). Performance of MDMN is also compared to illustrate the inadequacy of applying the domain-level adjustment alone. Lastly, we combine MDMN and CFR to form MDMNCFR. Though it adjusts at both cross-domain and between-treatment levels, its design is inferior to MDCN which is rigorously backed by supporting theories. Additional details can be found in Appendix \ref{app:model}.

\subsection{Synthetic Circular Data}
\label{sec:syn}

In this section, we create a synthetic case with shifts in both feature spaces and outcome functions across domains. We define the dimension of the feature space $\mathcal{X} \in \mathbb{R}^3$ and the number of domains $S=10$. These domains are labeled by $s\in\{0,\cdots, 9\}$. Any domain $s$ versus the rest of the combined domains can be regarded as the setup for source versus target domains. To simulate the data, we first use $s$ to create an angle parameter: $\angle_s=s\times\pi/10$, which evenly splits a circle. Through $\angle_s$, we vary domains and outcome functions with the following procedures.
\begin{equation*}
\begin{aligned}
&\text{Domain shift: } & X|s=\{x_1,x_2,x_3\} \sim \mathcal{N}([4\sin(\angle_s),4\cos(\angle_s),0]',\mathbb{I}_3); \\
&\text{Potential outcomes: } &Y(0)|X,s \sim \mathcal{N}(1.5[\sin(x_1+\angle_s)+\cos(x_2+x_3+\angle_s)],1),\\
&& Y(1)|X,s \sim \mathcal{N}(1.5[\cos(x_1+\angle_s)+\sin(x_2+x_3+\angle_s)],1);
\\
&\text{Treatment Assignment :}& \mathbb{P}(T=1|X)=1/[1+\exp(-0.5x_1-0.5x_2-2x_3)].
\end{aligned}
\end{equation*}
For source domains, the observed outcomes correspond to the outcomes with assigned treatments $Y=Y(0)(1-T)+Y(1)T$. We simulate 2,000 samples for each domain and repeat this procedure 10 times for variability assessment. The distributions for 10 domains in their first two dimensions are visualized in Figure \ref{fig:10dom}. The location shift of 10 domains forms a circle, and each domain $s$ has 2 close neighbors on its two sides with labels: $s-1\pmod{10}$ and  $s+1\pmod{10}$.

\begin{figure*}[t]
\centering
 \subfigure[Visualization of 10 Domains]{\label{fig:10dom}
    \includegraphics[height=.3\linewidth]{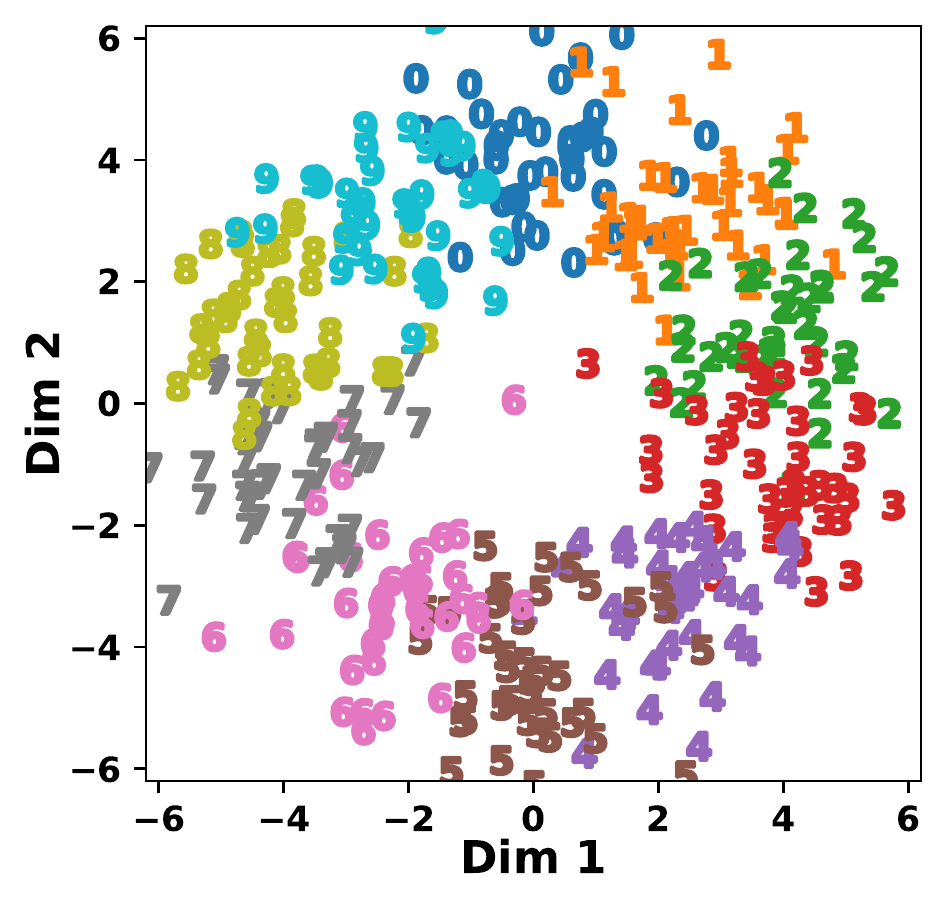}}
    \subfigure[CFR Domain Similarity]{\label{fig:cfr}
    \includegraphics[height=.3\linewidth]{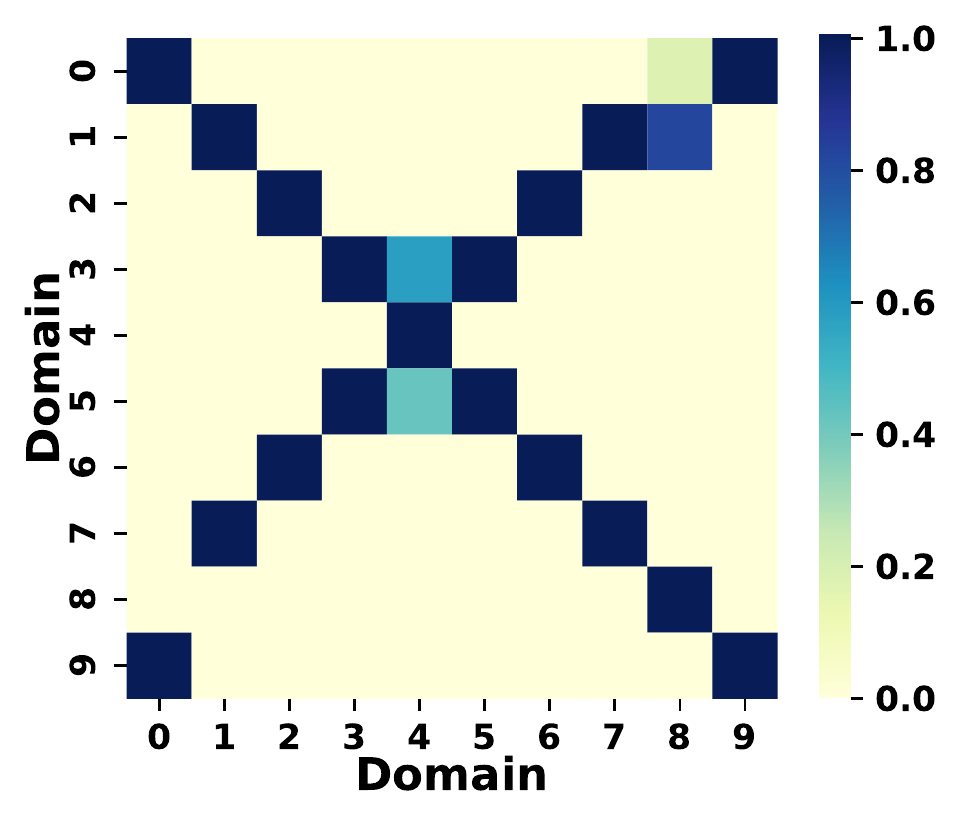}}
    \subfigure[MDMN Domain Similarity]{\label{fig:dmdcn}
    \includegraphics[height=.3\linewidth]{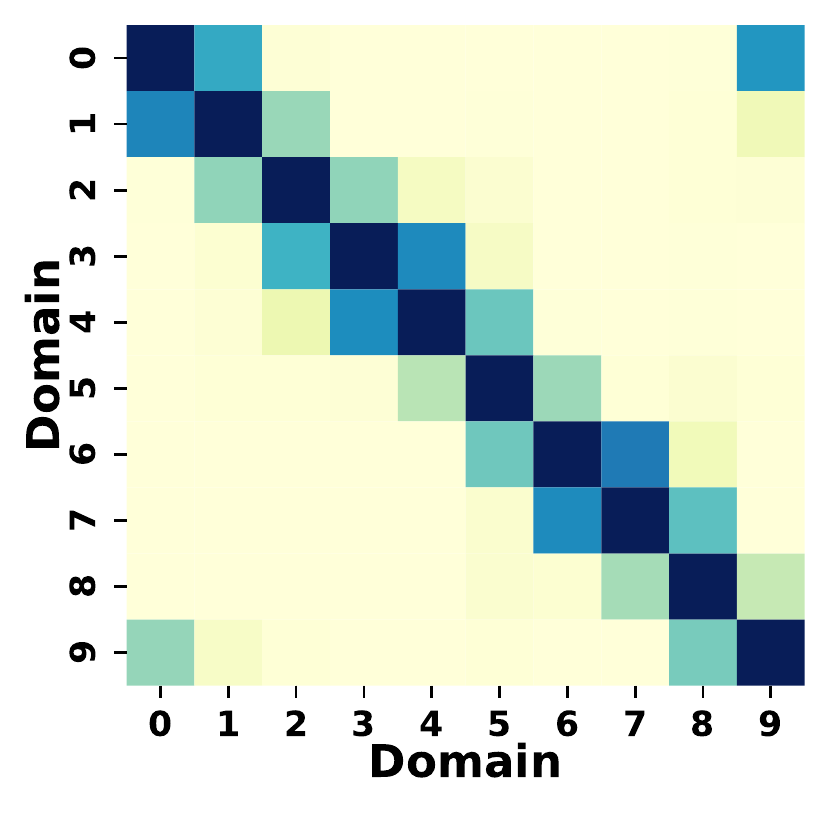}}
\caption{
\label{fig:domain_rela}
\ref{fig:10dom} visualizes the locations of the 10 domains. \ref{fig:cfr} and \ref{fig:dmdcn} are the heat maps measuring the similarity between domains. For MDMN-based approaches (MDCN included), we use the learned weight $(w_{s',s}+w_{s,s'})/2$ to fill in the grid$_{s,s'}$. For other methods, we use $\phi(x)$'s 1d t-SNE \citep{maaten2008visualizing} to estimate the distance $\boldsymbol{l_s}$ and weight $\boldsymbol{w_s}$ in \eqref{eq:rawdif} to plot the similarity (details in Appendix \ref{app:vis}). Darker colors indicate more similarity. Ideally, the whole diagonal and two off-diagonal corners should be dark with more decaying for grids further away from them. CFR strongly connects some distant domains, such as 1 and 7 or 2 and 6. Essentially, it reduces the distances in the diagonal to reduce the treatment group imbalance marginally, which contradicts with how similarity is shared between near domains. MDCN fully recovers the true similarity. Other methods are visualized in Figure \ref{fig:domain_rela2}.} 
\end{figure*}

\begin{table*}[t]
\centering
\caption{PEHE on Circular Data. MDMN is the most competitive of all benchmarks. Adjusting treatment group imbalance within each domain (DCFR) consistently outperforms CFR in all domains. Wilcoxon signed-rank test is used to report the significance between our newly proposed approaches and CFR. $P$ value $< 5 \cdot 10^{-2}$ ($\dagger$), $P$ value < $1 \cdot 10 ^{-4}$ ($\dagger \dagger$).}
\label{tab:pehe}
\resizebox{!}{0.145\textwidth}{
\begin{tabular}{l|c|c|c|c||c|c|c|}
\toprule
 {Target Domain} & {CF}&{MLP}&{CFR}& {MDMN}&{MDMNCFR$^{\dagger\dagger}$}&{DCFR$^{\dagger\dagger}$}&{MDCN$^{\dagger\dagger}$}
 \\\midrule   
0 & \textbf{1.24 $\pm$ .02} & 1.69 $\pm$ .08 & 1.59 $\pm$ .07 & 1.31 $\pm$ .05 & 1.33 $\pm$ .05 & 1.37 $\pm$ .07 & 1.28 $\pm$ .07 \\
1 & 2.55 $\pm$ .03 & 2.01 $\pm$ .05 & 1.99 $\pm$ .08 & 1.42 $\pm$ .04 & \textbf{1.34 $\pm$ .04} & 1.48 $\pm$ .09 & 1.37 $\pm$ .04 \\
2 & 1.43 $\pm$ .03 & 1.33 $\pm$ .06 & 1.32 $\pm$ .03 & \textbf{1.16 $\pm$ .04} & 1.17 $\pm$ .05 & 1.19 $\pm$ .03 & 1.18 $\pm$ .04 \\
3 & 1.44 $\pm$ .03 & 1.27 $\pm$ .04 & 1.25 $\pm$ .03 & 1.01 $\pm$ .04 & 1.00 $\pm$ .03 & 1.02 $\pm$ .02 & \textbf{0.99 $\pm$ .03} \\
4 & 1.45 $\pm$ .01 & 1.31 $\pm$ .06 & 1.29 $\pm$ .05 & 1.20 $\pm$ .02 & 1.23 $\pm$ .03 & 1.19 $\pm$ .04 & \textbf{1.16 $\pm$ .03} \\
5 & 1.52 $\pm$ .02 & 1.34 $\pm$ .07 & 1.25 $\pm$ .05 & 1.28 $\pm$ .05 & 1.31 $\pm$ .08 & 1.25 $\pm$ .06 & \textbf{1.24 $\pm$ .07} \\
6 & 1.41 $\pm$ .03 & 1.47 $\pm$ .07 & 1.32 $\pm$ .06 & 1.34 $\pm$ .07 & 1.26 $\pm$ .02 & \textbf{1.25 $\pm$ .03} & 1.32 $\pm$ .05 \\
7 & 2.20 $\pm$ .02 & 1.82 $\pm$ .08 & 1.87 $\pm$ .12 & 1.64 $\pm$ .10 & 1.80 $\pm$ .07 & \textbf{1.49 $\pm$ .07} & 1.60 $\pm$ .06 \\
8 & 1.69 $\pm$ .03 & 2.15 $\pm$ .05 & 2.10 $\pm$ .07 & 1.69 $\pm$ .04 & 1.71 $\pm$ .04 & 1.65 $\pm$ .07 & \textbf{1.62 $\pm$ .04} \\
9 & \textbf{1.49 $\pm$ .03} & 2.17 $\pm$ .07 & 2.31 $\pm$ .08 & 1.65 $\pm$ .06 & 1.59 $\pm$ .08 & 1.61 $\pm$ .09 & \textbf{1.49 $\pm$ .06} \\
Overall & 1.64 $\pm$ .04 & 1.66 $\pm$ .04 & 1.63 $\pm$ .04 & 1.37 $\pm$ .03 & 1.37 $\pm$ .03 & 1.35 $\pm$ .03 & \textbf{1.33 $\pm$ .03} \\
   \bottomrule
\end{tabular}
}
\end{table*}

The results on PEHE are summarized in Table \ref{tab:pehe}. For lower dimensional manifolds, tree-based approaches are advantageous as samples can easily saturate the space. However, CF lacks the structure to address the shifts in domains and outcome functions, so it only slightly outperforms MLP. Compared with MLP, other methods with adjustments demonstrate clear improvements in accuracy. DCFR minimizes the treatment imbalance within each domain instead of marginally like CFR. This lessens the impact of domain-level relationships, thus resulting in better performance. This result is visualized in Figure \ref{fig:cfr}, as CFR distorts some of the cross-domain similarities to alleviate the marginal treatment imbalance (e.g., regarding domain 2 and 6 to be closely connected). DCFR visualized in \ref{fig:domain_rela2} better maintains the original domain-level connectivity. MDCM in Figure \ref{fig:tmdcn} can clearly recover the true similarity.
The combination of CFR and MDMN (MDNMCFR) does not improve MDMN. MDCN has a clear advantage over all other benchmarks, achieving top marks in 5 out of 10 domains as well as  best overall performance.

\subsection{Semi-Synthetic Mouse Data}
\label{sec:micedt}

In this experiment, we make use of a mouse local field potential (LFP) dataset collected in \citet{Gallagher2017}. These data were collected from 21 mice in accordance with guidelines provided by the Institutional Animal Care and Use Committee (IACUC). LFPs were recorded while each mouse was exposed to two conditions: ``home cage'' (resting state, low stress) and ``open field'' (mild increased stress) \citep{bailey2009anxiety}. We use these two states to mimic the control (home cage) and treatment (open field) groups, thus providing an analog to multicenter observational study data. Individual mice, similarly to individual sites, are expected to demonstrate heterogeneity. We define an observation as each 1 second-long window of brain signal recordings which is associated with a continuous outcome, such as a behavioral outcome. Each mouse has less than 500 observations for each state.  This treatment structure mimics approaches in closed-loop brain stimulation, which chooses treatments based on current neural representations, and is currently in use in depression \cite{scangos2021closed} and seizure prevention \cite{stanslaski2012design}. 

\begin{table*}[t]
\centering
\caption{PEHE on Mice Data. MDCN gives the overall best performance, and is the best approach in 9 out of 21 domains. $P$ value $< 5 \cdot 10^{-2}$ ($\dagger$), $P$ value < $1 \cdot 10 ^{-4}$ ($\dagger \dagger$).}
\label{tab:pehe1}
\resizebox{!}{0.25\textwidth}{
\begin{tabular}{l|c|c|c|c||c|c|c|}
\toprule
 {Target Domain} &{CF}& {MLP}&{CFR}& {MDMN}&{MDMNCFR$^{\dagger}$}&{DCFR$^{\dagger\dagger}$}&{MDCN$^{\dagger\dagger}$}
 \\\midrule   
0 & 7.64 $\pm$ .07 & 1.61 $\pm$ .02 & 1.50 $\pm$ .02 & 1.52 $\pm$ .05 & 1.56 $\pm$ .04 & \textbf{1.45 $\pm$ .06} & 1.52 $\pm$ .06 \\
1 & 3.73 $\pm$ .06 & 4.21 $\pm$ .12 & 4.12 $\pm$ .12 &  3.31 $\pm$ .19 & 3.35 $\pm$ .09 & \textbf{3.24 $\pm$ .12} & 3.45 $\pm$ .08 \\
2 & 5.17 $\pm$ .04 & 5.42 $\pm$ .11 & 4.94 $\pm$ .13 & 3.29 $\pm$ .16 & 3.95 $\pm$ .26 & 3.70 $\pm$ .20 & \textbf{3.15 $\pm$ .23} \\
3 & 6.92 $\pm$ .04 & \textbf{2.07 $\pm$ .05} & 2.11 $\pm$ .05 & 2.01 $\pm$ .07 & 2.16 $\pm$ .05 & 2.10 $\pm$ .06 & 2.17 $\pm$ .07 \\
4 & 4.09 $\pm$ .12 & 2.61 $\pm$ .07 & \textbf{2.47 $\pm$ .06} & 2.57 $\pm$ .08 & 2.68 $\pm$ .10 & 2.65 $\pm$ .08 & 2.63 $\pm$ .07 \\
5 & 6.68 $\pm$ .04 & 1.49 $\pm$ .12 & 1.37 $\pm$ .08 & 1.39 $\pm$ .03 & 1.34 $\pm$ .05 & \textbf{1.31 $\pm$ .04} & 1.40 $\pm$ .07 \\
6 & 3.77 $\pm$ .05 & \textbf{2.19 $\pm$ .06} & 2.40 $\pm$ .14 & 2.26 $\pm$ .06 & 2.26 $\pm$ .06 & 2.33 $\pm$ .08 & 2.26 $\pm$ .08 \\
7 & 6.06 $\pm$ .21 & \textbf{4.85 $\pm$ .08} & 5.11 $\pm$ .07 & 5.20 $\pm$ .08 & 5.22 $\pm$ .10 & 5.27 $\pm$ .09 & 5.11 $\pm$ .05 \\
8 & 2.84 $\pm$ .05 & 2.27 $\pm$ .16 & 2.23 $\pm$ .10 & 1.71 $\pm$ .11 & 1.87 $\pm$ .17 & 1.73 $\pm$ .10 & \textbf{1.70 $\pm$ .09} \\
9 & 2.53 $\pm$ .11 & 1.88 $\pm$ .16 & 1.76 $\pm$ .15 & 2.04 $\pm$ .22 & 1.57 $\pm$ .14 & 1.48 $\pm$ .14 & \textbf{1.31 $\pm$ .07} \\
10 & 4.54 $\pm$ .25 & 1.72 $\pm$ .10 & 1.50 $\pm$ .08 &  1.46 $\pm$ .05 & 1.50 $\pm$ .03 & \textbf{1.40 $\pm$ .06} & \text{1.41 $\pm$ .03} \\
11 & 7.54 $\pm$ .16 & 1.66 $\pm$ .03 & 1.72 $\pm$ .08 & 1.44 $\pm$ .07 & 1.41 $\pm$ .05 & 1.38 $\pm$ .06 & \textbf{1.24 $\pm$ .04} \\
12 & \textbf{2.70 $\pm$ .14} & 2.81 $\pm$ .13 & 2.79 $\pm$ .08 & 2.90 $\pm$ .11 & 2.99 $\pm$ .08 & 2.87 $\pm$ .09 & 2.96 $\pm$ .10 \\
13 & 2.60 $\pm$ .06 & 1.67 $\pm$ .19 & \textbf{1.56 $\pm$ .12} & 1.98  $\pm$ .12 & 2.26 $\pm$ .18 & 1.91 $\pm$ .15 & 2.23 $\pm$ .24 \\
14 & 4.40 $\pm$ .42 & 1.48 $\pm$ .02 & \textbf{1.38 $\pm$ .02} & 1.39 $\pm$ .02 & 1.47 $\pm$ .02 & 1.43 $\pm$ .03 & \text{1.43 $\pm$ .03} \\
15 & 2.49 $\pm$ .15 & 3.20 $\pm$ .15 & 2.95 $\pm$ .13 & 2.38 $\pm$ .11 & \textbf{2.37 $\pm$ .10} & {2.47 $\pm$ .14} & 2.54 $\pm$ .11 \\
16 & 3.62 $\pm$ .17 & 2.66 $\pm$ .03 & 2.51 $\pm$ .13 & 2.60 $\pm$ .17 & 2.34 $\pm$ .09 & 2.46 $\pm$ .12 & \textbf{2.21 $\pm$ .09} \\
17 & 2.66 $\pm$ .44 & 1.68 $\pm$ .03 & 1.57 $\pm$ .02 & 1.59 $\pm$ .05 & \textbf{1.51 $\pm$ .02} & 1.55 $\pm$ .05 & \textbf{1.51 $\pm$ .03} \\
18 & 5.84 $\pm$ .17 & 6.14 $\pm$ .20 & 5.95 $\pm$ .16 & 4.90 $\pm$ .34 & 5.12 $\pm$ .23 & 5.17 $\pm$ .38 & \textbf{4.75 $\pm$ .27} \\
19 & 4.42 $\pm$ .24 & 2.03 $\pm$ .07 & 1.81 $\pm$ .07 & 1.41 $\pm$ .05 & 1.41 $\pm$ .05 & 1.38 $\pm$ .03 & \textbf{1.37 $\pm$ .05} \\
20 & 2.29 $\pm$ .15 & 2.49 $\pm$ .09 & 2.36 $\pm$ .12 & 2.38 $\pm$ .10 & \textbf{2.28 $\pm$ .09} & 2.29 $\pm$ .08 & \textbf{2.28 $\pm$ .09} \\
overall & 4.41 $\pm$ .13 & 2.77 $\pm$ .09 & 2.58 $\pm$ .09 & 2.37 $\pm$ .08 & 2.41 $\pm$ .08 & 2.36 $\pm$ .08 & \textbf{2.32 $\pm$ .07}\\
   \bottomrule
\end{tabular}
}
\end{table*}

\begin{figure*}[t]
\centering
 \subfigure[DCFR]{\label{fig:tdcfr}
    \includegraphics[height=.31\linewidth]{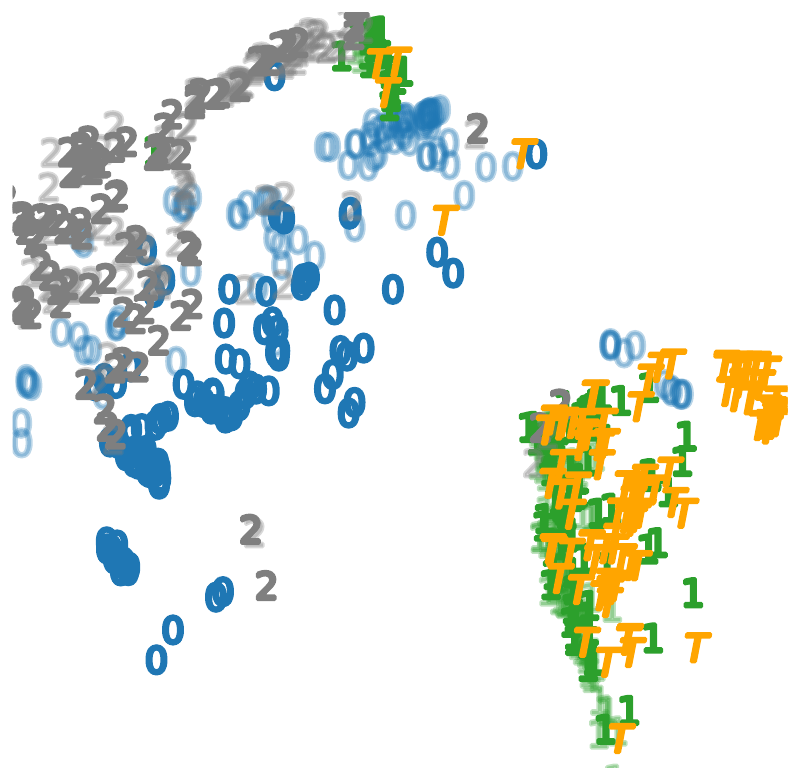}}
    \subfigure[MDMN]{\label{fig:tmdmn}
    \includegraphics[height=.31\linewidth]{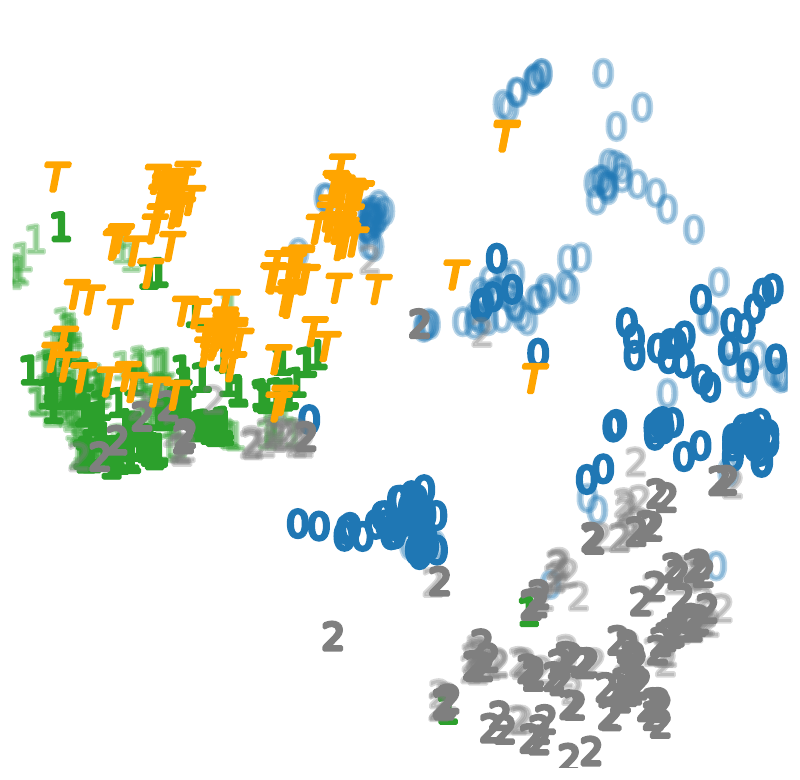}}
    \subfigure[MDCN]{\label{fig:tmdcn}
    \includegraphics[height=.31\linewidth]{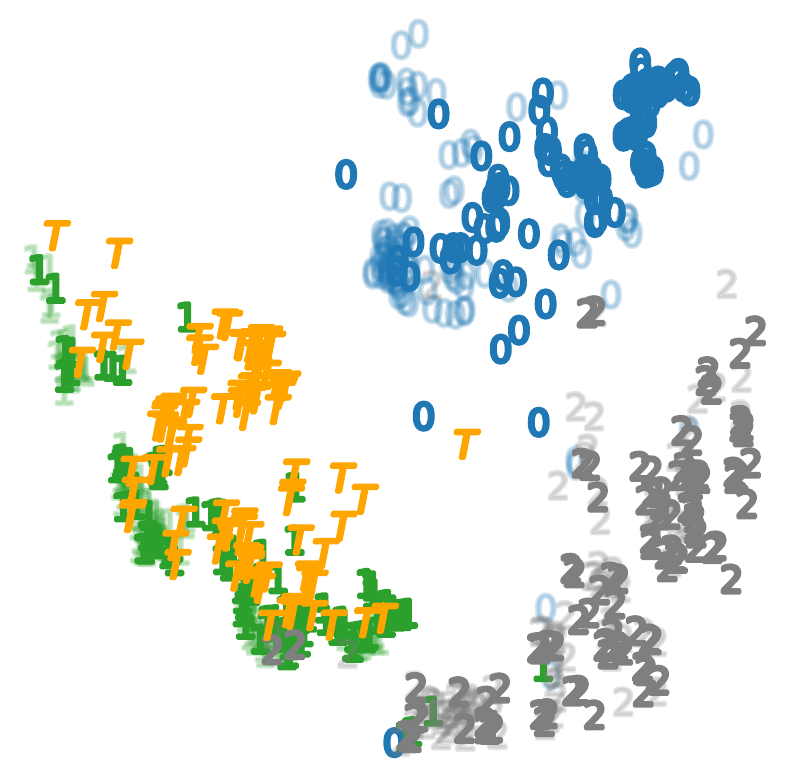}}
\caption{
\label{fig:tsne} T-SNE plot of $\phi(x)$. The domain 9 is chosen as the target domain and indicated with an orange "T". For convenience, the three nearest neighbors of the target domain are labeled as ``0'', ``1'', and ``2''. The weight matrix in MDCN and MDMN both identify the same three neighbors. Solid transparent shades represent the control and treatment groups, respectively. In \ref{fig:tdcfr}, the target domain lies on the outer perimeter of these three neighbors indicating lesser similarity. In \ref{fig:tmdmn}, there exists a visible separation between the treatment and the control groups in the neighbor ``0''. In \ref{fig:tmdcn}, MDCN balances the treatment and control groups with respects to its neighbors, and the target domain is properly surrounded by the three neighbors. Other visualizations are provided in Appendix \ref{app:vis}.}
\end{figure*}

Features meant to capture neural dynamics are generated from LFP data, including power spectral density \cite{Welch1967}, coherence, and Granger causality \cite{Granger1969} for a total of with 9,856 features. To artificially increase sample size, we employ an auto-encoder method that simultaneously interpolates the space between samples and reduces dimensionality \citep{pmlr-v139-oring21a}. Further experimental details are provided in Appendix \ref{app:preprocess}. Following preprocessing, each domain contains 20,000 samples with the feature space dimension reduced to a dimensionality 15. The new feature space still retains domain-specific information, as mouse identity can be accurately predicted ($>90\%$). Next, we simulate the potential outcomes. We randomly draw 15 correlated neural networks with the input and output size: $\mathbb{R}^{15} \rightarrow \mathbb{R}$, for both states, $\{f_1^0,\cdots,f_{15}^0\}$, $\{f_1^1,\cdots,f_{15}^1\}$, that are designed to vary with the center of each domain $s$, $\bar{x}^s=\{\bar{x}^s_1,\cdots,\bar{x}^s_{15}\}$ (details provided in Appendix \ref{app:preprocess}). Essentially, domains that are close in center tend to have similar outcome functions.  We simulate the potential outcomes for domain $s$ as follows, $Y(0)|X,s \sim \mathcal{N}(\sum\nolimits_{j=1}^{15} \sigma(\bar{x}^s_1)f_j^0(X),1)$; $Y(1)|X,s \sim \mathcal{N}(\sum\nolimits_{j=1}^{15} \sigma(\bar{x}^s_1)f_j^1(X),1)$.

We evenly split the data into 10 folds for variability assessment with results summarized in Table \ref{tab:pehe1}. The drawbacks of the tree-based CF approach are demonstrated as the increased feature dimension exacerbates decrease of domain overlap. MDCN demonstrates the best performance in 9 domains, outperforming benchmarks which conduct either domain-level adjustment (MDMN) or treatment-level adjustment (DCFR) but not both. DCFR outperforms CFR with statistical significance, which further supports adjusting for treatment imbalance differentially in different domains. This is also reflected in the combination of CFR and MDMN (MDMNCFR), which results in a drop in performance when compared to MDMN. We visualize the learned feature embeddings $\phi(x)$ in Figure \ref{fig:tsne2}, shedding light on how MDCN predicts on unlabeled target domains. We first find the three closest ``neighborhoods'' to the target domain using the MDCN weight matrix and then plot their relative positions. The target domain in MDCN is located at the center of its three neighbors in Figure \ref{fig:tmdcn}, which enables it to borrow information from their respective outcome predictions. Additionally, the treatment and control groups from the three neighbors are all well-balanced, making MDCN more robust against treatment group imbalances. In Figure \ref{fig:tdcfr}, DCFR regards two of its neighbors as irrelevant domains, preventing it from learning from more reference cases. In Figure \ref{fig:tmdmn}, the treatment group imbalance is not well controlled in MDMN, thus resulting in less competitive generalization to the target domain.

\section{Discussion}
\label{sec:dis}


Here, we propose MDCN, an approach that provides robust estimation of the conditional treatment effect for data collected from multi-center observational studies with particular emphasis in accurately inferring treatment effects from individuals from a new center.  Our approach addresses a lack of methods research with respects to this specific but not uncommon problem. More importantly, our model demonstrates the potential to provide patients from new, unobserved centers with better references for treatment. Additionally, we also elaborate on the design of MDCN and underpin it with supporting theory. In empirical evaluations, MDCN consistently outperforms state-of-the-art domain adaptation methods and causal methods. One limitation of MDCN lies in the scalability with respects to domain size $S$ which comes at a cost of $\mathcal{O}(S^2)$. To address this scalability issue, one could combine similar domains via some similarity heuristic. Likewise, scalability becomes an issue when handling multiple or continuous treatment regimes. As the number of treatment groups increases, learning a balanced treatment embedding tends to be more challenging. Adequately addressing these issues provides interesting directions for future research in this area. Moving forward, our techniques may be applied to multicenter observational studies and neuromodulation; the ethics of neuromodulation are complex, and we do not view our study as directly impacting these discussions.  Finally, as with any developed methodology with relevance to healthcare, the methods proposed here should not be solely relied upon to guide clinical practice, suggest treatments, or recommend any health-related actions without consultation and close collaboration with medical professionals.

\section*{Acknowledgments}
Research reported in this manuscript was supported by the National Institute of Biomedical Imaging and Bioengineering and the National Institute of Mental Health through the National Institutes of Health BRAIN Initiative under Award Number R01EB026937 and the National Institute Of Mental Health of the National Institutes of Health under Award Number R01MH125430. The content is solely the responsibility of the authors and does not necessarily represent the official views of the National Institutes of Health.

\bibliographystyle{plainnat}
\bibliography{refer}

\section*{Checklist}

\begin{enumerate}

\item For all authors...
\begin{enumerate}
  \item Do the main claims made in the abstract and introduction accurately reflect the paper's contributions and scope?
    \answerYes{}
  \item Did you describe the limitations of your work?
    \answerYes{In Section \ref{sec:dis}.}
  \item Did you discuss any potential negative societal impacts of your work?
    \answerYes{In Section \ref{sec:dis}.}
  \item Have you read the ethics review guidelines and ensured that your paper conforms to them?
    \answerYes{}
\end{enumerate}

\item If you are including theoretical results...
\begin{enumerate}
  \item Did you state the full set of assumptions of all theoretical results?
    \answerYes{In Sections \ref{sec:statement}, \ref{sec:mdcn} and \ref{asec:proof}.}
        \item Did you include complete proofs of all theoretical results?
    \answerYes{In Section \ref{asec:proof}.}
\end{enumerate}

\item If you ran experiments...
\begin{enumerate}
  \item Did you include the code, data, and instructions needed to reproduce the main experimental results (either in the supplemental material or as a URL)?
    \answerYes{In the supplemental material.}
  \item Did you specify all the training details (e.g., data splits, hyperparameters, how they were chosen)?
    \answerYes{In Section \ref{app:model}.}
        \item Did you report error bars (e.g., with respect to the random seed after running experiments multiple times)?
    \answerYes{In Table \ref{tab:pehe} and \ref{tab:pehe1}.}
        \item Did you include the total amount of compute and the type of resources used (e.g., type of GPUs, internal cluster, or cloud provider)?
    \answerYes{In Section \ref{app:model}.}
\end{enumerate}

\item If you are using existing assets (e.g., code, data, models) or curating/releasing new assets...
\begin{enumerate}
  \item If your work uses existing assets, did you cite the creators?
    \answerYes{In Section \ref{app:model} and \ref{sec:micedt}.}
  \item Did you mention the license of the assets?
    \answerYes{In Section \ref{app:model} and \ref{sec:micedt}.}
  \item Did you include any new assets either in the supplemental material or as a URL?
    \answerYes{In the supplemental material.}
  \item Did you discuss whether and how consent was obtained from people whose data you're using/curating?
    \answerNA{The assets were approved by Institutional Animal Care and Use Committee (IACUC).} 
  \item Did you discuss whether the data you are using/curating contains personally identifiable information or offensive content?
    \answerNA{}
\end{enumerate}

\item If you used crowdsourcing or conducted research with human subjects...
\begin{enumerate}
  \item Did you include the full text of instructions given to participants and screenshots, if applicable?
    \answerNA{}
  \item Did you describe any potential participant risks, with links to Institutional Review Board (IRB) approvals, if applicable?
    \answerNA{}
  \item Did you include the estimated hourly wage paid to participants and the total amount spent on participant compensation?
    \answerNA{}
\end{enumerate}

\end{enumerate}


\clearpage
\appendix

\setcounter{figure}{0}
\makeatletter 
\renewcommand{\thefigure}{S\@arabic\c@figure}
\makeatother
\setcounter{table}{0}
\makeatletter 
\renewcommand{\thetable}{S\@arabic\c@table}
\makeatother
\setcounter{proposition}{0}
\makeatletter 
\renewcommand{\theproposition}{S\@arabic\c@proposition}
\makeatother
\setcounter{assumption}{0}
\makeatletter 
\renewcommand{\theassumption}{S\@arabic\c@assumption}
\makeatother
\setcounter{claim}{0}
\makeatletter 
\renewcommand{\theclaim}{S\@arabic\c@claim}
\makeatother
\setcounter{lemma}{0}
\makeatletter 
\renewcommand{\thelemma}{S\@arabic\c@lemma}
\makeatother

\section{Proofs for Theoretical Analysis in Section \ref{sec:theory}}
\label{asec:proof}
We heuristically explain in the main article that the proposed method is inspired by an error bound on the target domain. Here, we provide details in its derivation. To make this section self-dependent,
we repeat some definitions, and further comment on their implied properties.

\textbf{Definition 1} (Probabilistic Discrepancy).
For two hypotheses $h,h':\mathbb{R}^P\rightarrow \mathbb{R}$, their difference based on a probabilistic distribution $D$ over $\mathcal{X}$ is defined as  $\gamma(h,h'|D)=\mathbb{E}_{x\sim D}|h(x)-h'(x)|$.

Definition \ref{def:pdis} is used to quantify the distance between any two hypotheses based on a distribution $D$. This discrepancy is symmetric as $\gamma(h,h'|D)=\gamma(h',h|D)$. If we replace $h$ by the ground truth outcome function, it represents the error of $h'$.

\textbf{Definition 2} (Lipschitz Continuity).
A function $f :\mathbb{R}^P\rightarrow\mathbb{R}$ is Lipschitz continuous with parameter $\lambda$, if 
$|f(x_1)-f(x_2)|\leq \lambda ||x_1-x_2||_2$
holds for any vectors $x_1,x_2\in\mathcal{X}$. We denote the family as  $\mathcal{F}_\lambda$.

In this family, for any $\lambda$, $f \in \mathcal{F}_\lambda$ $\implies -f \in \mathcal{F}_\lambda$. We assume that the proposed hypothesis functions $\{h_0, h_1\} \in F_\lambda$ which can be represented by neural networks, and the true functions $\{g_{s,0},g_{s,1}\}_{s=1}^{S} \in F_{\lambda^*}$.

The family of Lipschitz functions with parameter $\lambda$ can also measure the Wasserstein distance between spaces. Its relationship to the Wasserstein-1 distance is described below.

\begin{equation}
\label{eq:wasslike}
    \sup\limits_{f: \mathcal{X} \rightarrow \mathcal{R},||f||_L<\lambda} \mathbb{E}_{x\sim D}[f(x)]-\mathbb{E}_{x\sim D'}[f(x)]=\lambda W_1(D,D')
\end{equation}

\begin{lemma}[Additivity of Lipschitz functions]
If $f_1 \in F_{\lambda_1}$ and $f_2 \in F_{\lambda_2}$, then $f_1+f_2 \in F_{\lambda_1+\lambda_2}$.
\end{lemma}
This can be shown trivially with the triangle inequality.  By the symmetry of the Lipschitz family, 
$f_1-f_2,-f_1+f_2,-f_1-f_2 \in F_{\lambda_1+\lambda_2}$.


\begin{lemma}[Symmetry of the Wasserstein distance]
For any two probabilistic distributions $D$ and $D'$ over $\mathcal{X}$,
$W_1(D,D')=W_1(D',D)$.
\end{lemma}

\begin{proof}
Assume that $f' \in F_1$ is the function that maximizes $\mathbb{E}_{x\sim D}[f(x)]-\mathbb{E}_{x\sim D}[f(x)]$. We have $\mathbb{E}_{x\sim D}[f'(x)]-\mathbb{E}_{x\sim D}[f'(x)]=W_1(D,D')$. Since $f \in \mathcal{F}_1 \implies -f'\in \mathcal{F}_1$,

\begin{equation*}
 W_1(D,D')=\mathbb{E}_{x\sim D'}[-f'(x)]-\mathbb{E}_{x\sim D}[-f'(x)]\le  W_1(D',D).
\end{equation*}
From the other direction, we can similarly show that  
$W_1(D',D)\ge W_1(D,D')$, which leads to $W_1(D',D)=W_1(D,D')$.

\end{proof}

To bound the error on the target domain, we take a two-step procedure. First, we bound the error of the unobserved potential outcomes within each domain given the observed data. Then we explore the cross-domain relationships to bound the overall error on the target domain with labeled data from source domains.

\textbf{Between-treatment Bound}

For any source domain $s$ with $D_s$, we want to minimize the distance between our proposed hypotheses $h_0,h_1$ and $g_{s,0},g_{s,1}$. Since we are only able to observe $T=0$ and $T=1$ in $D_{s,0}$, and $D_{i,1}$, respectively. The full potential outcome error on $D_s$ cannot be directly calculated.  
Instead, Proposition \ref{prop:bt} suggests that this error can be bounded.

\begin{proposition}
\label{prop:bt}
For any source domain $s$ with the marginal probabilities of receiving treatment and control as $p_s^{T=1}$ and $p_s^{T=0}$,  the probabilistic discrepancy between $\{h_0,h_1\}$ and $\{g_{s,0},g_{s,1}\}$,  $\gamma(h_0,g_{i,0}|D_s)+\gamma(h_1,g_{i,1}|D_s)$ has an upper bound,

\begin{equation*}
\small
\begin{split}
\gamma(h_0,g_{s,0}|D_s)+\gamma(h_1,g_{s,1}|D_s)\le 
\gamma(h_0,g_{s,0}|D_{s,0})+\gamma(h_1,g_{i,1}|D_{s,1})+{(\lambda+\lambda^*)}W_1(D_{s,0},D_{s,1}).
\end{split}
\end{equation*}

\end{proposition}

\begin{proof}
For ease of derivation, we only expand on the error bound of the control group $T=0$,  $\gamma(h_0,g_{s,0}|D_s)$. In T-learner, $h_0$ and $h_1$ are constructed separately, and they do not influence each other.  The bound on $T=1$ can be derived with identical steps. 

\begin{equation*}
\small
    \begin{split}
    \gamma(h_0,g_{s,0}|D_s)&=\mathbb{E}_{x\sim D_s}|h_0(x)-g_{s,0}(x)|\\
    &=p_s^{T=0}\mathbb{E}_{x\sim D_{s,0}}|h_0(x)-g_{s,0}(x)|+p_s^{T=1}\mathbb{E}_{x\sim D_{s,1}}|h_0(x)-g_{s,0}(x)|\\
    &= p_s^{T=0}\mathbb{E}_{x\sim D_{s,0}}|h_0(x)-g_{s,0}(x)|+
    p_s^{T=1}\mathbb{E}_{x\sim D_{s,0}}|h_0(x)-g_{s,0}(x)|\\
    & \quad\quad\quad\quad - p_s^{T=1}\mathbb{E}_{x\sim D_{s,0}}|h_0(x)-g_{s,0}(x)|+
    p_s^{T=1}\mathbb{E}_{x\sim D_{s,1}}|h_0(x)-g_{s,0}(x)|\\
    &=\gamma(h_0,g_{s,0}|D_{s,0})+p_s^{T=1}\biggl(\mathbb{E}_{x\sim D_{s,1}}|h_0(x)-g_{s,0}(x)|-\mathbb{E}_{x\sim D_{s,0}}|h_0(x)-g_{s,0}(x)|\biggl)\\
    & \le \gamma(h_0,g_{s,0}|D_{s,0})+p_s^{T=1}\biggl( \sup\limits_{f: \mathcal{X} \rightarrow \mathcal{R},||f||_L<\lambda+\lambda^*} \mathbb{E}_{x\sim D_{i,0}}[f(x)]-\mathbb{E}_{x\sim D_{s,1}}[f(x)]\biggl)\\
    &=\gamma(h_0,g_{s,0}|D_{s,0})+p_s^{T=1}(\lambda+\lambda^*)W_1(D_{s,0},D_{s,1})
    \end{split}
\end{equation*}

In the fifth line, we use the additivity of the Lipschitz family, $h_0-g_{s,0} \in \mathcal{F}_{\lambda+\lambda^*}$.
Similarly, between $h_1$ and $g_{s,1}$, we have 
\begin{equation*}
\gamma(h_1,g_{s,1}|D_i)\le \gamma(h_1,g_{s,1}|D_{s,1})+p_s^{T=0}(\lambda+\lambda^*)W_1(D_{s,1},D_{s,0}).    
\end{equation*}

With the symmetry of $W(\cdot , \cdot)$, $(\lambda+\lambda^*)W_1(D_{s,1},D_{s,0})=(\lambda+\lambda^*)W_1(D_{s,0},D_{s,1})$. Then we sum up the two groups, and the bound is proven.

\end{proof}

\textbf{Cross-Domain Bound}

In this step, we simplify the problem to a multiple domain setups but assuming randomization within each domain,
$D_{s,0}=D_{s,1}=D_{s}$. We keep the other conditions unchanged such as the domain shift and the shift in outcome functions. As the outcomes on the target domain are unobservable, the error can still only be bounded rather than calculated. Proposition \ref{prop:cd} suggests a bound.

Before we state the bound, we define two quantities.
We assume that $w=\{w_s\}_{s=1}^{S-1}$ are weights on labeled source domains and they sum up to 1. $\gamma^*_0$ and $\gamma^*_0$ define the minimum discrepancies attainable for the weighted summation of all domains.

\begin{equation}
\label{eq:gammamin}
\begin{split}
&\gamma^*_0=\min\limits_{h_0}\biggl[\gamma(h_0,g_{S,0}|D_S)+\sum\limits_{s=1}^{S-1}w_s\gamma(h_0,g_{s,0}|D_s)\biggl]; \\
&\gamma^*_1=\min\limits_{h_1}\biggl[\gamma(h_1,g_{S,1}|D_S)+\sum\limits_{s=1}^{S-1}w_s\gamma(h_1,g_{s,1}|D_s)\biggl].
\end{split}
\end{equation}

In \eqref{eq:gammamin}, the two quantities depict the fundamental difference in true outcome functions across all domains, which is unoptimizable. The larger its value, the harder the target domain error can be controlled.
If $g_{s,0},g_{s,1}$ only differ slightly across domains, we are more likely to find some universally defined $h_0^*$ and  $h_1^* \in \mathcal{F}_{\lambda^*}$ that make $\gamma^*_0$ and $\gamma^*_1$ small.

\begin{proposition}
\label{prop:cd}
Assume that we have full randomization within each domain, and
$h_0^*$ and  $h_1^* \in \mathcal{F}_{\lambda^*}$ with minimum errors $\gamma^*_0$ and $\gamma^*_1$ in \eqref{eq:gammamin}. For any positive weights on source domains $w=\{w_s\}_{s=1}^{S-1}$ with $\sum_{s=1}^{S-1} w_s=1$. The target domain error can be bounded by,
 
 \begin{equation}
 \label{eq:bdrand}
     \begin{split}
     \gamma(h_0,g_{S,0}|D_S)+\gamma(h_1,g_{S,1}|D_S)&\le
     2(\lambda+\lambda^*)W_1(D_S, \sum\limits_{s=1}^{S-1}w_sD_s)+\gamma_0^*+\gamma_1^*\\
     &+
     \sum\limits_{s=1}^{S-1}w_s[\gamma(h_0,g_{0,s}|D_s)+\gamma(h_1,g_{1,s}|D_s) ].
     \end{split}
 \end{equation}
\end{proposition}

\begin{proof}
Again, we prove it on the control group and extend it to both groups with the property of ``T-learner".

\begin{equation*}
\resizebox{.99\textwidth}{!}{%
\begin{math}
    \begin{aligned}
        \gamma(h_0,g_{S,0}|D_S)&\le \gamma(h_0,h_0^*|D_S)+\gamma(g_{S,0},h_0^*|D_S) \text{~ ( Triangle Inequality )}\\
        &=\gamma(h_0,h_0^*|D_S)+\gamma(g_{S,0},h_0^*|D_S)+\sum\limits_{s=1}^{S-1}w_s\gamma(h_0,h_0^*|D_s)-\sum\limits_{s=1}^{S-1}w_s\gamma(h_0,h_0^*|D_s)\\
        &\le\gamma(h_0,h_0^*|D_S)+[\gamma(g_{S,0},h_0^*|D_S)+\sum\limits_{s=1}^{S-1}w_s\gamma(g_{s,0},h_0^*|D_s)]\\&\quad\quad\quad\quad\quad\quad+\sum\limits_{s=1}^{S-1}w_s\gamma(h_0,g_{s,0}|D_s)-\sum\limits_{s=1}^{S-1}w_s\gamma(h_0,h_0^*|D_s)\\
        &= \gamma(h_0,h_0^*|D_S)-\sum\limits_{s=1}^{S-1}w_s\gamma(h_0,h_0^*|D_s)+\gamma_0^* + \sum\limits_{s=1}^{S-1}w_s\gamma(h_0,g_{0,s}|D_s) \\
        & \le (\lambda+\lambda^*)W_1(D_S, \sum\limits_{s=1}^{S-1}w_s D_s)+\gamma_0^*+\sum\limits_{s=1}^{S-1}w_s\gamma(h_0,g_{0,s}|D_s) ~~ ( \text{as } h_0-h_0^* \in \mathcal{F}_{\lambda+\lambda^*} )\\ 
    \end{aligned}
\end{math}
}
\end{equation*}

Likewise, we can obtain the following quantity for the treatment group.
\begin{equation*}
 \gamma(h_1,g_{S,1}|D_S) \le (\lambda+\lambda^*)W_1(D_S, \sum\limits_{s=1}^{S-1}w_sD_s)+\gamma_1^*+\sum\limits_{s=1}^{S-1}w_s\gamma(h_1,g_{1,s}|D_s)  
\end{equation*}
Lastly, we sum over two groups to get the bound in \eqref{eq:bdrand}.
\end{proof}

\textbf{Overall Bound}

Proposition  \ref{prop:bt} or Proposition \ref{prop:cd} only partially address our big goal. However, their combination gives us the overall bound on the target domain in a multicenter observational study setup.

\textbf{Theorem 2}
For any positive weights on source domain $w=\{w_s\}_{s=1}^{S-1}$ with $\sum_{s=1}^{S-1} w_s=1$. The overall target domain error for the proposed hypothesis functions $\{h_0,h_1\}$ can be bounded by,
 \begin{equation*}
 \small
     \begin{split}
     \gamma(h_0,g_{S,0}|D_S)+\gamma(h_1,g_{S,1}|D_S)&\le
     (\lambda+\lambda^*)[2W_1(D_S, \sum\limits_{s=1}^{S-1}w_{s}D_s)+\sum\limits_{s=1}^{S-1}w_{s}W_1(D_{s,0},D_{s,1})] \\
     &+
     \sum\limits_{s=1}^{S-1}w_{s}[\gamma(h_0,g_{s,0}|D_{s,0})+\gamma(h_1,g_{i,1}|D_{s,1})]+\gamma_0^*+\gamma_1^*.
     \end{split}
 \end{equation*}

\begin{proof}
We replace each $\gamma(h_0,g_{0,s}|D_s)$ and $\gamma(h_1,g_{1,s}|D_s)$ in Proposition \ref{prop:cd} with the bound derived in Proposition \ref{prop:bt}, and the overall bound can be obtained. 
\end{proof}

\section{Method Implementations and Data Preprocessing}

\subsection{Model Specifications}
\label{app:model}

\begin{algorithm}[ht]
	\caption{MDCN Algorithm \label{alg}}
	\begin{algorithmic}
 \STATE [\textbf{Input}:] Data with treatment labels and observed outcomes, $\{\{t_i, x_i, y_i\}_{i=1}^{N_s}\}_{s=1}^{S-1}$ from source domains and unlabeled data from the target domain $\{x_i\}_{i=1}^{N_S}$. \\
\STATE [\textbf{Output}:] Feature embedding function $\phi$, potential outcome models $\{h_0,h_1\}$, domain discriminator $f_{cd}$ and treatment discriminator $f_{bt}$ .\\\hrulefill
		\FOR{iter $= 1$ to $n^{iter}$}
		\STATE Sample a mini-batch from source domains  $\{\{t_i, x_i, y_i\}_{i=1}^{N_s}\}_{s=1}^{S-1}$ and target domain $\{x\}_{i=1}^{N_S}$. 
		\STATE Calculate $\{\boldsymbol{l}_s\}_{s=1}^S$ and update $\{\boldsymbol{w}_s\}_{s=1}^S$ in \eqref{eq:rawdif} for cross-domain adjustment.
		    \FOR{iter $= 1$ to $n_{1}$}
            \STATE Sample a mini-batch from source domains  $\{\{t_i, x_i, y_i\}_{i=1}^{N_s}\}_{s=1}^{S-1}$ and target domain $\{x\}_{i=1}^{N_S}$.
        \STATE Optimize full loss in \eqref{eq:fullloss} with respect to the parameters in $\phi, h_0, h_1$.
        \STATE Update the parameters in $\phi, h_0, h_1$.
        \ENDFOR\\
        \FOR{iter $=1$ to $n_{2}$}
        \STATE Sample a mini-batch from source domains  $\{\{t_i, x_i, y_i\}_{i=1}^{N_s}\}_{s=1}^{S-1}$ and target domain $\{x\}_{i=1}^{N_S}$.
        \STATE Optimize the cross-domain  $-L_{cd}(\phi,f_{cd})$ in \eqref{eq:cd} with respect to the parameters in $f_{cd}$. 
        \STATE Update the parameters in $f_{cd}$.
        \STATE Optimize the between-treatment  $-L_{bt}(\phi,f_{bt})$ in \eqref{eq:bt} with respect to the parameters in $f_{bt}$. 
        \STATE Update the parameters in $f_{bt}$.
    \ENDFOR\\
            \ENDFOR\\
	\end{algorithmic}
\end{algorithm}

For MLP, CFR, MDMN, MDNMCFR and MDCN, we use the same model architecture for a fair comparison. The feature embedding network is parametrized through a neural network with two hidden layers of 50 units each. The outcome networks for both treatment and control groups are based on a neural network with two hidden layers of 50 units. For the circular data example and the mice data example, we set the dimensions of the embedded space to 10 and 20, respectively. The Lipschitz constraints for the Wasserstein-1 distance is maintained through the gradient penalty \citep{gulrajani2017improved}. The tuning parameters have fixed values, as $\alpha$=1e-3 and $\beta$=5e-4. In practice, these specifications could also be tuned. We use ELU as the activation function  \citep{djork2016fast}  and ADAM as the optimizer \citep{kingma2015adam} with step size 1e-4 throughout both experiments.  The full algorithm of MDCN is sketched out in Algorithm \ref{alg}.
The code will be public on Github with the MIT license when the manuscript is accepted.

The method CF is implemented through the R package grf with GPL-3 license \citep{grf}. We let the model automatically tune its hyper-parameters by specifying \verb|tune.parameters="all"|.

All Python- based methods are run on a single NVIDIA P100 GPU; the R-based CF is run on an Intel(R) Xeon(R) Gold 6154 CPU.

\subsection{Preprocessing the mice data}
\label{app:preprocess}

The raw data have high variability, as the first 100 principle components only explains 50\% of the total variance. Additionally, the limited sample size poses a challenge with regards to the variability assessment.
To enlarge the sample size and decrease the variability of the data. We use a recently published auto-encoder for data interpolation \citep{pmlr-v139-oring21a}.  The encoder $f_e: \mathbb{R}^{9856}\rightarrow \mathbb{R}^{15}$ is a neural network that has three hidden layers of 100 units each. 
The decoder $f_d: \mathbb{R}^{15}\rightarrow \mathbb{R}^{9856}$ is a neural network with one hidden layer of 50 units. Then, we construct a discriminative network: $f_i:\mathbb{R}^{15}\rightarrow [0,1]$ with a single hidden layer of 50 units to judge the qualify of interpolated data. Lastly, we add a
 network with no hidden layer to predict mice identity $f_m: \mathbb{R}^{15}\rightarrow [0,1]^{21}$.

For each domain $s$,  we have the following loss for the control group (home cage),
\begin{equation*}
\resizebox{.98 \textwidth}{!}{
\begin{math}
\begin{aligned}
L_{s,0}=&\mathbb{E}_{x_1,x_2\sim{D_{s,0}},\alpha\sim U(0,1)}\min_{f_e,f_d,f_m} \max_{f_i} \biggl\{ \\ &\biggl[(f_d(f_e(x_1))-x_1)^2+(f_d(f_e(x_2))-x_2)^2\biggl]+ ~\hspace{7cm} \text{(I)}\\
&.1\biggl[f_e(f_d(\alpha f_e(x_1)+(1-\alpha)f_e(x_2)))-(\alpha f_e(x_1)+(1-\alpha)f_e(x_2))\biggl]^2+ ~\hspace{3.3cm} \text{(II)}\\
&.05\biggl[{CE}(f_m(f_e(x_1)),s))+ {CE}(f_m(f_e(x_2)),s))+{CE}(f_m(\alpha f_e(x_1)+(1-\alpha)f_e(x_2),s))\biggl]+ ~\hspace{.2cm} \text{(III)}\\
&.05\biggl[\log(f_i(f_e(x_1)))+ \log(f_i(f_e(x_2)))+\log(1-f_i(\alpha f_e(x_1)+(1-\alpha)f_e(x_2)))\biggl]\biggl\}. ~\hspace{1.5cm} \text{(IV)}
\end{aligned}
\end{math}
}
\end{equation*}
We have the data reconstruction in (I). 
(II) is the ``cycle loss" \citep{pmlr-v139-oring21a}, which is to reconstruct the interpolated feature embedding. It also contributes to make $f_e$ and $f_d$ as reciprocal of each other. (III) is the cross-entropy loss for each domain to be predictive of its domain identity $s$. It is helpful in maintaining domain-specific information and creating domain shift.
(IV) is the adversarial loss. It discriminates how similar interpolated points are to the raw data. Enforcing it enables the interpolation to retain the original data distributions.
The full loss is calculated by summing over all domains and all treatment groups,  $\sum_{s=1}^{21}[L_{s,0}+L_{s,1}]$.
We train this framework for 12,000 iterations. In each iteration, we randomly draw 20 samples from each treatment group and each domain, which makes a total batch size of 840. After the training is finished, we observe that the embedded space explains >30\% of the total variance.
we randomly interpolate 10,000 samples for each domain and each treatment group to make an enlarged dataset. 
These interpolated points overall reaches >90\% accuracy in predicting the mouse identity, with  $>60\%$ confidence indicated by the discriminative network $f_i$ that they are from the true data.

For outcome models, we prepare 30 mini neural networks with input and output size: $\mathbb{R}^{15} \rightarrow \mathbb{R}$,  $\{f_1^0,\cdots,f_{15}^0\}$, $\{f_1^1,\cdots,f_{15}^1\}$.
These neural networks all have a single hidden layer with 15 units, and hyperbolic tangent activation functions.
Their weights are initialized through the Xavier initialization \cite{glorot2010understanding}. 
Then we use 15 evenly spaced points from $[-1.5,1.5]$ to shift their weight parameters to make them distinct.
Explicitly, we shift the weight parameters in $f_s^0 \in \{f_1^0,\cdots,f_{15}^0\}$  by $+(3s/15-1.5)$. For $\{f_1^1,\cdots,f_{15}^1\}$, we shift them by another permutation of these 15 points.
Lastly, they are combined with the the center of each domain $s$, $\bar{x}^s=\{\bar{x}^s_1,\cdots,\bar{x}^s_{15}\}$ to create the overall shift in outcome functions. It encourages domains that are close in center to have similar outcome functions,
\begin{equation*}
\begin{aligned}
Y(0)|X,s \sim \mathcal{N}(\sum\limits_{j=1}^{15}sigmoid(\bar{x}^s_1)f_j^0(X),1); ~~ Y(1)|X,s \sim \mathcal{N}(\sum\limits_{j=1}^{15}sigmoid(\bar{x}^s_1)f_j^1(X),1).
\end{aligned}
\end{equation*}

\section{Additional Visualizations}
\label{app:vis}
\begin{figure*}[t]
\centering
 \subfigure[MDMN]{\label{appfig:MDMN}
    \includegraphics[height=.3\linewidth]{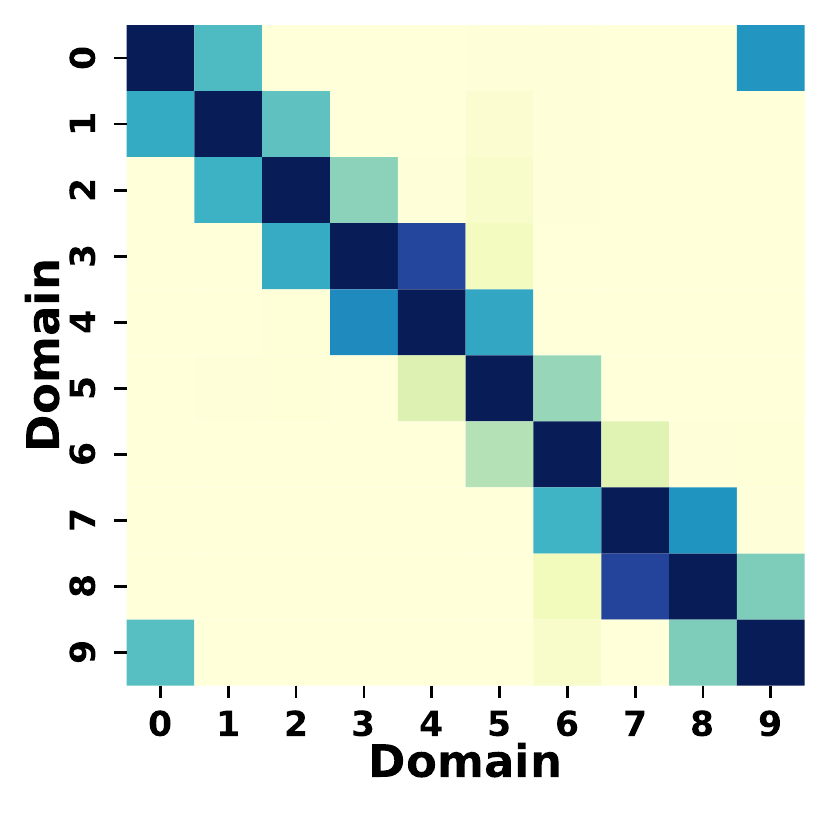}}
    \subfigure[MDMNCFR ]{\label{appfig:MDMNCFR}
    \includegraphics[height=.3\linewidth]{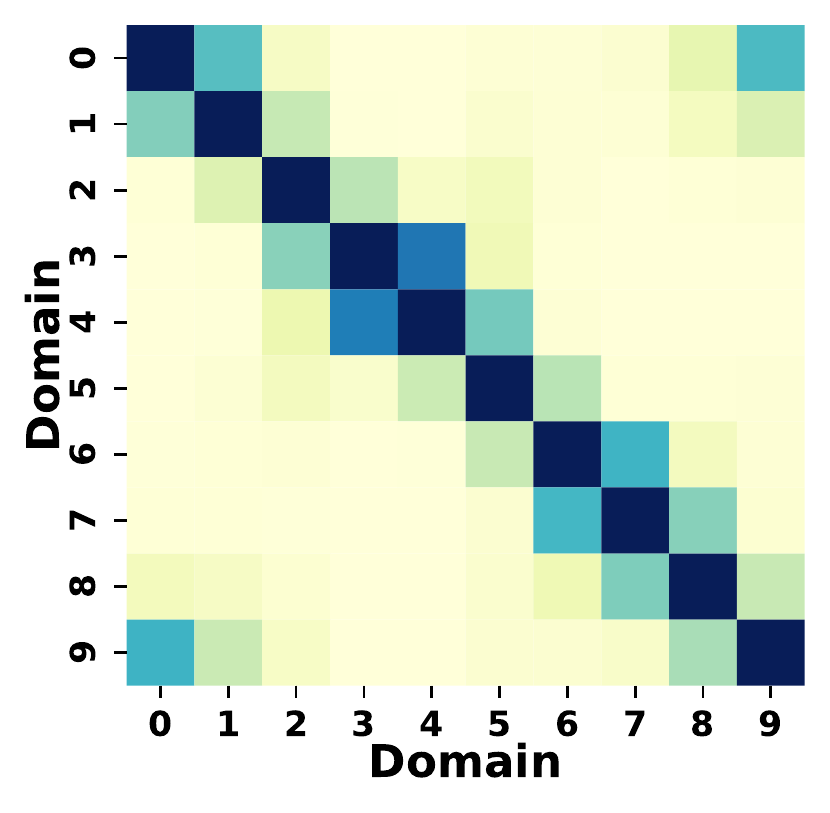}}
    \subfigure[DCFR ]{\label{appfig:dcfr}
    \includegraphics[height=.3\linewidth]{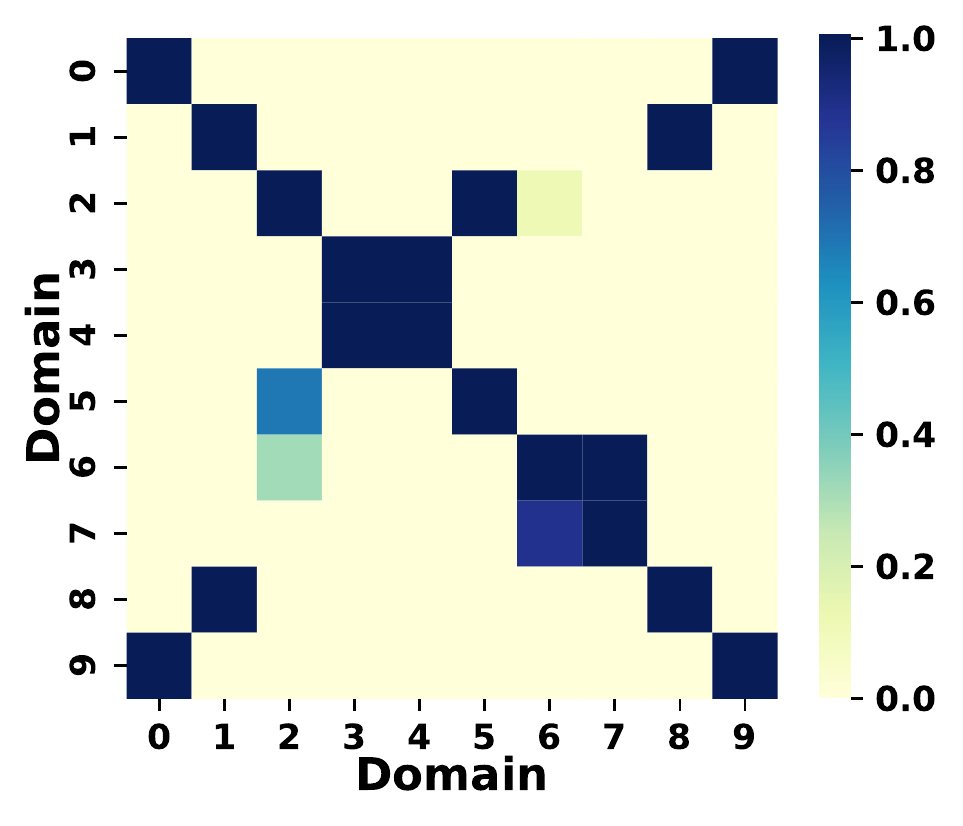}}
\caption{
\label{fig:domain_rela2}
Additional visualizations of domain-level similarities. DCFR in \ref{appfig:dcfr} better maintains the relative closeness across domains than CFR does in Figure \ref{fig:3d}. However, it is still not appropriately capturing the original domain-level similarity as shown in \ref{appfig:MDMN} and \ref{appfig:MDMNCFR}, both of which are backed by the cross-domain adjustment.
}
\end{figure*}

\begin{figure*}[ht]
\centering
 \subfigure[CFR]{\label{fig:tcfr}
    \includegraphics[height=.4\linewidth]{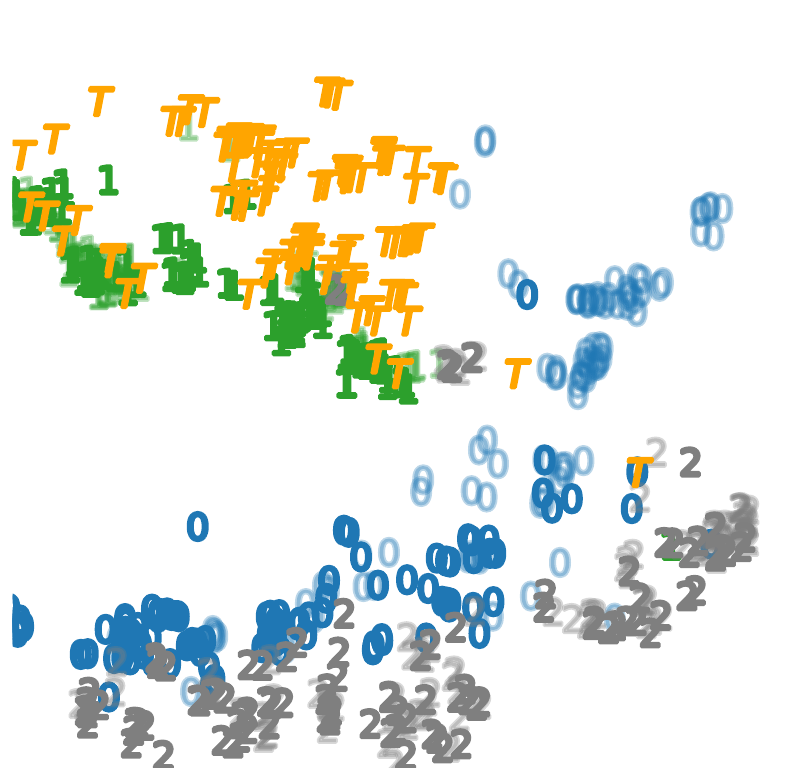}}
    \subfigure[MDMNCFR]{\label{fig:tmdmncfr}
    \includegraphics[height=.4\linewidth]{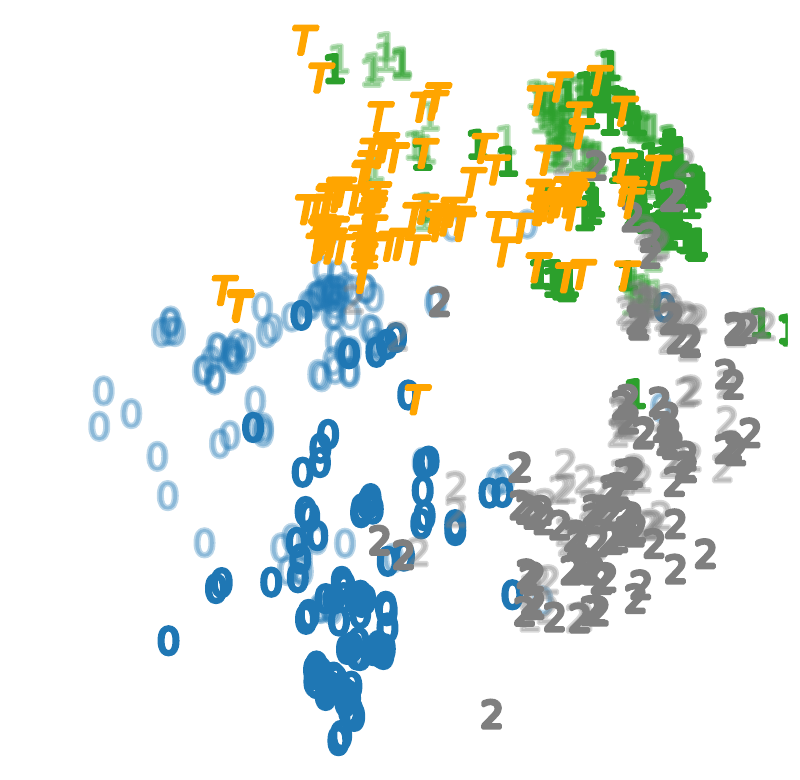}}
\caption{
\label{fig:tsne2} Additional t-SNE plots for learned $\phi(x)$. The domain 9 is chosen as the target domain and is labeled as orange "T". 
For convenience, its 3 nearest neighboring domains are labeled as "0", "1" and "2". They are found through the 
weight matrix in MDMN. Labels with solid shade and transparent shade represent the control group and the treatment group, respectively. }
\end{figure*}

Figure \ref{fig:domain_rela2} visualizes the similarity at domain level for the other 3 methods.
For methods without the domain level adjustment (CFR included), we use a surrogate to calculate 
$\boldsymbol{l}_s$ in \eqref{eq:rawdif}, so that similarity could be represented in a comparable manner. We first use 1-d t-SNE on the learned feature embedding $\phi(x)$, as $\text{tsne}(\phi(x))$. Then we calculate the distance and weight according to \eqref{eq:rawdif2}, 
\begin{equation}
\label{eq:rawdif2}
\begin{split}
&\boldsymbol{l}_s=\{l_{s,i}\}_{i\ne s},~
l_{s,i}=\mathbb{E}_{x \sim {D}_{s}}[\text{tsne}(\phi(x))]-\mathbb{E}_{x \sim {D}_{i}}[\text{tsne}(\phi(x))],\\
&\boldsymbol{w}_s=\{w_{s,i}\}_{i\ne s}=\text{softmax}(-\boldsymbol{l}_s).
\end{split}
\end{equation}
Diagonal elements are set to 1 as reference values. DCFR visualized in Figure \ref{fig:domain_rela2} better maintains the original domain-level connectivity than CFR in Figure \ref{fig:domain_rela}. It connect such as domains 3 and 4 or domains 6 and 7.
MDMN and MDMNCFR both have the cross-domain component to learn domain similarities. They also clearly reflect the original data distributions by showing strong connectivity to two neighboring domains.

Figure  \ref{fig:tsne2} includes two additional t-SNE plots of the learned feature $\phi(x)$ in mouse data. In \ref{fig:tcfr}, the target domain is distant to the neighboring domains "0" and "2", which gives the model fewer relevant points to learn from to predict on the target domain.
In \ref{fig:tmdmncfr}, the treatment group and control group are still less balanced compared to that in Figure \ref{fig:tmdcn}, this suggests worse generalization from our bound in Theorem \ref{thm:bound}.

\end{document}